\documentclass{article}

\usepackage{fullpage}

\usepackage[utf8]{inputenc} 
\usepackage[T1]{fontenc}    
\usepackage{hyperref}       
\usepackage{url}            
\usepackage{booktabs}       
\usepackage{amsfonts}       
\usepackage{nicefrac}       
\usepackage{microtype}      
\usepackage{amsmath}
\usepackage{amssymb}
\usepackage{mathtools}
\usepackage{amsthm}
\usepackage[capitalize,noabbrev]{cleveref}      
\usepackage{lipsum}         
\usepackage{graphicx}
\usepackage{natbib}
\usepackage{doi}

\usepackage{adjustbox}
\usepackage{dsfont}
\usepackage{float}
\usepackage{subcaption}
\usepackage{enumitem}

\newcommand{\argmin}[1]{\underset{#1}{\mathrm{argmin}} \ }
\newcommand{\argmax}[1]{\underset{#1}{\mathrm{argmax}} \ }

\theoremstyle{plain}
\newtheorem{theorem}{Theorem}[section]
\newtheorem{proposition}[theorem]{Proposition}
\newtheorem{lemma}[theorem]{Lemma}
\newtheorem{corollary}[theorem]{Corollary}
\theoremstyle{definition}
\newtheorem{definition}[theorem]{Definition}
\newtheorem{assumption}[theorem]{Assumption}
\theoremstyle{remark}

\usepackage{xcolor}
\usepackage{tcolorbox}
\definecolor{regcol}{HTML}{45347f}
\definecolor{volcol}{HTML}{20908c}
\definecolor{IICcol}{HTML}{9ad83c}
\newcommand{\tbox}[2]{%
  \begingroup\setlength\fboxsep{1.2pt}\colorbox{#1}{#2}\endgroup%
}

\newcommand{\mathbox}[2]{%
  \begingroup\setlength\fboxsep{1.2pt}%
  \colorbox{#1}{$\displaystyle \vphantom{\frac12}#2$}%
  \endgroup%
}
\newcommand{\reg}[1]{\mathbox{regcol!20}{\;#1\;}}
\newcommand{\vol}[1]{\mathbox{volcol!20}{\;#1\;}}
\newcommand{\dd}{\mathrm{d}}

\title{A Regularization-Sharpness Tradeoff for Linear Interpolators}

\date{}

\author{%
          Qingyi Hu \\
  School of Mathematics and Statistics \\
  University of Melbourne, Australia \\
  \texttt{qingyih2@student.unimelb.edu.au}\\
  \and
  Liam Hodgkinson \\
  School of Mathematics and Statistics\\
  University of Melbourne, Australia\\
  \texttt{lhodgkinson@unimelb.edu.au}\\
 }

\hypersetup{
pdftitle={A Regularization--Sharpness Tradeoff for Linear Interpolators},
pdfauthor={Qingyi Hu, Liam Hodgkinson},
}

\begin{document}
\maketitle

\begin{abstract}
The rule of thumb regarding the relationship between the bias-variance tradeoff and model size plays a key role in classical machine learning, but is now well-known to break down in the overparameterized setting as per the double descent curve. In particular, minimum-norm interpolating estimators can perform well, suggesting the need for new tradeoff in these settings. 
Accordingly, we propose a regularization-sharpness tradeoff for overparameterized linear regression with an $\ell^p$ penalty. Inspired by the interpolating information criterion, our framework decomposes the selection penalty into a regularization term (quantifying the alignment of the regularizer and the interpolator) and a geometric sharpness term on the interpolating manifold (quantifying the effect of local perturbations), yielding a tradeoff analogous to bias–variance. Building on prior analyses that established this information criterion for ridge regularizers, this work first provides a general expression of the interpolating information criterion for $\ell^p$ regularizers where $p\ge 2$. Subsequently, we extend this to the LASSO interpolator with $\ell^1$ regularizer, which induces stronger sparsity. Empirical results on real-world datasets with random Fourier features and polynomials validate our theory, demonstrating how the tradeoff terms can distinguish performant linear interpolators from weaker ones.
\end{abstract}


\section{Introduction}
\emph{Statistical learning} is concerned with building predictors from limited data, and the key measure of goodness is how well a model generalizes to unseen inputs. Classical intuition is often summarized by the \emph{bias-variance tradeoff}, which suggests a U-shaped risk curve as the number of parameters grows, and motivates model selection tools such as the Akaike Information Criterion (AIC)~\citep{akaike_new_1974} and the Bayesian Information Criterion (BIC)~\citep{schwarz_estimating_1978}. 

Deep learning often operates far beyond this regime: overparameterized models can achieve (near) zero training error and still generalize well, and test error can exhibit \emph{double descent} around the interpolation threshold~\citep{zhang_understanding_2021, rolnick_deep_2018, belkin_reconciling_2019, nakkiran_deep_2019}.  Perfectly interpolating solutions in the overparameterized regime can still generalize effectively in certain settings. Classical information criteria break down under interpolation, motivating the search for new model selection principles in the overparameterized setting.

\begin{figure}
    \centering
    \includegraphics[width=0.55\linewidth]{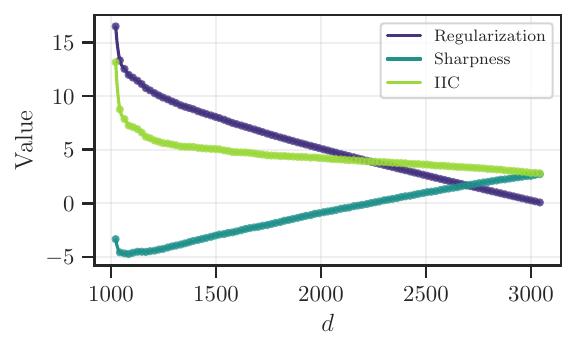}
    \caption{Decomposition of the Interpolating Information Criterion (green) for minimum $\ell^3$-norm interpolating solutions using random Fourier features as a \emph{tradeoff} between the effect of regularization (purple) and local sharpness (blue).}
    \label{fig:iic_intro}
\end{figure}

 Another prominent recent line of work explores conditions for \emph{benign overfitting}  and locates the interpolation peak in several models~\citep{bartlett_benign_2020, gunasekar_implicit_2017, wang_tight_2021}. The most significant of these developments for this work is the Interpolating Information Criterion (IIC)~\citep{hodgkinson_interpolating_2023}, which offers a view of information criteria that aligns with overparameterized settings. By leveraging a Bayesian duality principle, the marginal likelihood associated with the initial complex, high-dimensional model becomes equivalent to that of a lower-dimensional dual model. In the interpolation limit, this scenario simplifies considerably, and the marginal likelihood that is commonly used to measure model performance becomes more straightforward to approximate. 
 The IIC provides a principled way to compare overparameterized models while conditioning on fitting to the data exactly. Through the PAC-Bayesian framework, the IIC can also be readily tied to the more traditional notion of test risk ~\citep{hodgkinson_pac-bayesian_2023}. In doing so, the explicit factors influencing generalization error take shape, offering a clearer picture that is reminiscent of prior analyses. \emph{Sharpness}  is a popular type of measure describing the effect of local perturbations and commonly appears in bounds of this nature \citep{neyshabur_exploring_2017,keskar2016large}. It is known to correlate well with performance \citep{jiang2019fantastic}, although it does not tell the whole story \citep{dinh2017sharp}. Bias-variance-like tradeoffs also appear in a looser sense in benign overfitting bounds \citep{bartlett_benign_2020}, but fail to exhibit the clarity seen in precise asymptotics for ridge regression \citep{hastie_surprises_2022,mei_generalization_2020}. This is unfortunate, as neural networks may exhibit behavior more reminiscent of sparsity-inducing LASSO-type estimators \citep{geifman2020similarity}. 

Building on these ideas and the IIC, we present an explicit \emph{regularization-sharpness tradeoff} for linear interpolators with $\ell^p$ regularizers, as a replacement guiding principle to the classical bias--variance tradeoff in the overparameterized regime. This tradeoff relies on a decomposition of the IIC into two competing terms: (i) a regularizer term measuring alignment between the solution and an imposed prior belief, and (ii) a geometric sharpness term governed by the data geometry and the local shape of the loss landscape. An example highlighting the behavior of these terms in model size is shown in Figure \ref{fig:iic_intro}. Although the linear model is far less general than the IIC framework allows for\footnote{The original work of \citet{hodgkinson_interpolating_2023} applies to virtually all large-scale neural network models in practice.}, our primary objective is to highlight the generality of this tradeoff across regularizers (see Table \ref{tab:tradeoffs-summary} for a summary), and to explore how  the effect of the regularizer influences model performance through this lens. Assessing this factor turns out to be difficult, even in the setting of linear models. 
Our analysis derives the IIC expression for a general $\ell^p$ regularizer with $p\ge 2$, and then for the case $p=1$. The significant difference between these two cases is that the polyhedral geometry of $\ell^1$ induces \emph{sparsity}, but also breaks most of the original techniques used to derive the IIC in \citet{hodgkinson_interpolating_2023}. Consequently, an entirely new analysis is required. 

\clearpage
\paragraph{Contribution}

Our primary contribution is the theoretical and empirical demonstration of the following:
\begin{tcolorbox}
In the overparameterized regime, linear interpolators exhibits a tradeoff between
\begin{itemize}[leftmargin=*]
\item \vspace{-.1cm}\reg{\textbf{Regularization}}: quantifies the alignment of the regularizer and the interpolator (type of bias); and
\item \vspace{-.1cm}\vol{\textbf{Sharpness}}: quantifies the effect of local perturbations, including data, optimized parameters, and ambient sharpness (type of variance).
\end{itemize}
\end{tcolorbox}
As part of our analysis, we present:
\begin{itemize}
    \item A general expression of the \textbf{interpolating information criterion} with $\ell^p$ regularizers where $p\ge 2$ and extends this theory to the sparse case $p=1$.
    \item Empirical results with random Fourier features and polynomials demonstrate how the tradeoff terms can distinguish performant linear interpolators from weaker ones.
\end{itemize}

In Section \ref{sec:Background} we review background on generalization and information criteria, and in Section \ref{sec:Prelim} we set up the overparameterized linear model and its associated problem, and motivate the marginal likelihood as a measure of model quality. Section \ref{sec:IICLinear} presents the IIC formulas for $\ell^p$ penalties and the resulting regularization–sharpness decomposition. Finally, Section \ref{sec:Experiments} reports experiments using random Fourier features and polynomial regression.

\section{Background and Related Work}
\label{sec:Background}
\paragraph{Classical vs. Interpolating Information Criteria.} Classical criteria like AIC and BIC evaluate models based on likelihood, incorporating a penalty for complexity to improve model selection \citep{akaike_new_1974, schwarz_estimating_1978}. In a regular, correctly specified parametric family with fixed dimension and \emph{nonsingular} Fisher information, a Laplace approximation shows that BIC matches the negative log marginal likelihood (see Section \ref{sec:Prelim}) up to an additive constant. The penalty is proportional to the nominal parameter count and sample size, and is often interpreted as an \emph{Occam factor} that balances goodness of fit against model complexity \citep{mackay1992bayesian}. 
Larger models, such as deep networks, mixture models, and overparameterized regressions tend to be singular \citep{wei2022deep}: the likelihood has flat directions, the Fisher information is rank deficient, and the posterior is not asymptotically normal. The effective complexity differs from the parameter count, so classical BIC provides a poor approximation of Bayes evidence. Singular learning theory and criteria such as WAIC \citep{vehtari_practical_2017} and WBIC \citep{watanabe_widely_nodate} adjust penalties using likelihood geometry invariants, but still rely on large data, small model size limits that also fail in the overparameterized setting.
A recent development that directly addresses this gap is the Interpolating Information Criterion (IIC) \citep{hodgkinson_interpolating_2023}. Starting from the marginal likelihood, using a principle of \emph{Bayesian duality} that rewrites the evidence of a high-dimensional interpolating model as a lower-dimensional dual integral, one can condition on an exact fit to the data to yield an estimate of model evidence that can compare overparameterized models and regularizers even when classical Laplace arguments fail. Although IIC is defined through marginal likelihood, it also admits a learning-theoretic interpretation. Further work has highlighted that a PAC-Bayes risk bound can be written as a constant plus a complexity term that matches the IIC objective up to scaling and lower-order terms \citep{hodgkinson_pac-bayesian_2023}. This link is helpful because it frames interpolating evidence not only as a Bayesian score, but also as an explicit control on test risk.

\vspace{-.25cm}
\paragraph{Generalization Properties.}
Classical learning theory controls the gap between population risk and empirical risk by a complexity term such as the Vapnik--Chervonenkis dimension \citep{vapnik_uniform_2015} or Rademacher complexity \citep{bartlett_rademacher_2001}. These quantities reinforce the rule-of-thumb that one should balance underfitting and overfitting through a U–shaped bias–variance tradeoff \citep{hastie_elements_2009}. Once models can interpolate the training data, empirical error is essentially zero while standard complexity measures grow monotonically with model size. Consequently, uniform convergence bounds become vacuous even though large interpolating models can still generalize well in practice \citep{dziugaite2017computing}.

One response is to analyze the test risk directly in simple overparameterized settings. In random features regression and related high–dimensional linear models, precise asymptotic formulas for the prediction error have been derived as sample size and dimension grow, showing how the risk depends on the spectrum of the data covariance and on the ratios between sample size, ambient dimension and number of features \citep{mei_generalization_2020, hastie_surprises_2022}. These formulas exhibit the famous \emph{double descent curve}, in which test error peaks near the interpolation threshold and can decrease again in the highly overparameterized regime~\citep{belkin_reconciling_2019, nakkiran_deep_2019}. The precise shape of the curve is not especially critical in light of the fact that practical models can exhibit any variety of different risk curves \citep{chen2021multiple}. Instead, double descent highlights that risk can be controlled even when the model has many more parameters than data.

\begin{table*}[t]
\centering
\caption{Summary of Tradeoffs across Measures of Model Performance (Term 1 vs. Term 2)}
\label{tab:tradeoffs-summary}
\begin{adjustbox}{max totalsize={\textwidth}{\textheight},center}
\begin{tabular}{@{} l l l p{0.56\textwidth} @{}}
\toprule
\textbf{Name} & \textbf{Tradeoff} & \textbf{Tradeoff Term 1} & \textbf{Tradeoff Term 2} \\
\midrule
MSE & bias vs.\ variance &
$(\mathbb{E}\hat f(x)-f(x))^2$ &
$\mathbb{E}[f(\theta^\ast)-\mathbb E[f(\theta^\ast)]]^2$ \\

BIC & likelihood vs.\ penalty &
$-2\log L(\hat\theta)$ &
$d\log n$ \\

IIC ($\ell^2$) & regularization vs. sharpness &
$2\log\|\theta^\ast\|_2$ &
$\frac{1}{n}\log\det(XX^\top) -\log n$ \\

IIC ($\ell^{p}$, $p\ge2$) & regularization vs. sharpness &
$\frac{2d-p(d-n)}{n}\log\|\theta^{*}\|_{p}$ &
$\frac{1}{n}\log \det(XX^\top)+\frac{p-2}{n}\sum_{j=1}^d\log|\theta^\ast_j|+ K_1(p,d,n)$ \\

IIC ($\ell^{1}$) & regularization vs.\ sharpness &
$2\log\|\theta^{*}\|_{1}$ &
$\frac{1}{n}\log\det(XX^{\top})-\frac{2}{n}\log(V_0) +K_2(d,n)$ \\
\bottomrule
\end{tabular}
\end{adjustbox}
\end{table*}

A second line of work gives sharp non-asymptotic bounds for minimum–norm interpolators in linear regression. \citet{bartlett_benign_2020} characterize when the minimum $\ell^2$ norm interpolating predictor has near–optimal risk, in terms of two effective rank quantities of the covariance operator, and show that benign overfitting occurs when there is a long flat tail in the spectrum so that noise can be spread across many nearly null directions rather than concentrated in a few important ones. Their bounds naturally decompose into a term controlled by how the signal aligns with a low–dimensional leading subspace and a term controlled by the amount of sharpness in high–dimensional subspaces available for absorbing noise, already hinting at a regularization-geometry tradeoff. Extensions of this analysis to ridge regression show that adding very small explicit $\ell^2$ regularization can preserve benign behavior under weaker spectral conditions \citep{tsigler_benign_nodate}.
\vspace{-.25cm}

\paragraph{Sparse Interpolators.}
More recently, attention has turned to sparse interpolators and $\ell^1$ regularization. \citet{wang_tight_2021} obtain matching upper and lower bounds of order $\sigma^2/\log (d/n)$ for the prediction error of the minimum $\ell^1$ norm interpolator, or basis pursuit, under isotropic Gaussian designs and suitable sparsity conditions, establishing that exact interpolation can still be statistically consistent in high dimensions. \citet{ju_overfitting_2020} study basis pursuit in finite samples and show that the error of the minimum $\ell^1$ norm interpolator can exhibit its own double descent behavior as the feature–to–sample ratio varies, in a way that reflects the polyhedral geometry induced by $\ell^1$ constraints and differs from minimum $\ell^2$ norm interpolation. Taken together, these generalization results suggest a common pattern: risk is controlled by an interplay between a norm of the chosen interpolating solution and geometric quantities derived from the data covariance or feature design. The present work builds on this pattern and seeks to make the same regularization–sharpness tradeoff explicit at the level of interpolating information criteria. In doing so, we also aim to explain why the $\ell^1$ case fundamentally differs in its behavior with respect to model size.

\section{Notation and Preliminaries}
\label{sec:Prelim}
Let $(x_i, y_i)_{i=1}^n$ be independent and identically distributed samples from an unknown data-generating distribution $\mathcal{D}$ on $\mathbb{R}^d \times \mathbb{R}$. We denote by
$X = (x_1^\top,\dots,x_n^\top)^\top \in \mathbb{R}^{n \times d}
$ and $Y = (y_1,\dots,y_n)^\top \in \mathbb{R}^n$ the design matrix and response vector, respectively. A (scalar-valued) predictor is a function $f:\mathbb{R}^d \to \mathbb{R}$; its performance is measured by the population risk (generalization error)
\[
L(f) = \mathbb{E}_{(x,y)\sim \mathcal{D}} \ell(f(x),y),
\]
where $\ell(\hat{y}, y)$ is a loss function. Since $\mathcal{D}$ is unknown, one often works with the empirical risk
\[
\ell_n(f) = \frac{1}{n}\sum_{i=1}^n \ell(f(x_i),y_i).
\]
Throughout this paper, we use the squared loss $\ell(\hat{y},y) = (y - \hat{y})^2$, so that $\ell_n(f)$ corresponds to the training mean squared error (MSE). In the interpolating regime, many predictors achieve (near) zero empirical risk, and model comparison therefore relies on differences in the population risk, rather than on training performance. 

\subsection{Linear Regression}

We specialize to linear predictors $f_\theta(x) = x^\top \theta$ with parameter $\theta \in \Theta \subseteq \mathbb{R}^d$ and consider the Gaussian linear model $Y = X \theta + \epsilon$ where $\epsilon \sim \mathcal{N}(0, \gamma/2 I_n)$. Under this model\footnote{The precise reason for this parameterization is to ensure that the likelihood becomes equivalent to the Gibbs likelihood for the mean-squared error loss with temperature $\gamma > 0$}, the likelihood function is given by
\begin{align*}
p(Y|X,\theta) &= (\pi \gamma)^{-\frac{n}{2}}\exp\left (-\frac{1}{\gamma}\sum_{i=1}^n (y_i-x_i^\top \theta)^2\right)\\
&= c_{n,\gamma} e^{-\frac{1}{\gamma} \ell_n(\theta)},
\end{align*}
where $\ell_n(\theta)$ is the empirical risk in the predictor $f_\theta$ and $c_{n,\gamma} = (\pi \gamma)^{-\frac{n}{2}}$ is the normalizing constant for the likelihood. In the overparameterized regime $d > n$, there are infinitely many maximum likelihood estimators (all of which achieve zero loss on the training set), and so an additional criterion is required to select a unique predictor. To select among interpolating solutions, we adopt a Bayesian formulation with an $\ell^p$-type regularization induced by a generalized Gaussian prior
\[
\pi(\theta)
= c_{p,\tau}
\exp\left(-\frac{1}{\tau}\|\theta\|_p^p\right),\;
c_{p,\tau}^{-1}
= 2^d \Gamma\!\left(\frac{1}{p}+1\right)^d \tau^{\frac{d}{p}},
\]
where $\tau>0$ controls the strength of regularization. For $p=2$, this reduces to the standard Gaussian prior, while other values of $p$ promote different notions of simplicity.


\subsection{Importance of Marginal Likelihood}
The marginal likelihood is defined as the normalized constant of the posterior distribution
\begin{align*}
Z_n(M) = p(Y\mid X, M) &= \int_\Theta \pi(\theta\mid M)p(Y\mid X,\theta,M)\dd \theta\\
&= \int_{\mathbb{R}^d} \pi(\theta)c_{n,\gamma}e^{-\frac{1}{\gamma}\ell_n(\theta)}\dd \theta.
\end{align*}
A literal interpretation of $Z_n$ is an aggregation of the fits of all parameter values weighted by the prior. For more complex predictors $\hat{f}$, computing $Z_n$ rapidly becomes intractable, especially when the number of parameters $d$ is large. While computing $Z_n$ is often unnecessary for estimation, it is valued for model comparison. Suppose that $p(M \mid Y,X)$ denotes the probability that the model $M$ is ``accurate'' given data $X, Y$. From Bayes' Theorem,
\[
p(M\mid Y,X)\propto Z_n(M) p(M),
\]
where $p(M)$ is a prior probability on the accuracy of the model $M$. Provided that models should be considered equally viable \emph{a priori}, $p(M)$ is constant, and model selection becomes equivalent to identifying the model $M$ with the largest $Z_n(M)$. 

One further application is \emph{model averaging}, whereby a more stable prediction for a quantity of interest $\Delta$ can be made if one has a collection of models $M_i$ \emph{and} their marginal likelihoods~\citep{hoeting_bayesian_nodate}. Indeed, taking $p(M)$ to be constant, we have 
\[
p(\Delta\mid X, Y) \propto \sum_{i=1}^m p(\Delta\mid X,Y,M_i)Z_n(M_i).
\]
Accurate evaluation of $Z(M)$ matters not only for selecting a single model, but also for producing stable predictions that account for model uncertainty.



\subsection{Generalization Error and PAC-Bayes}

Despite its intrinsic interpretation as a measure of model quality, the marginal likelihood is not a popular object in the machine learning literature \citep{lotfi_bayesian_2023}. Instead, direct estimates of the population risk, including the test loss, are much more common. To connect generalization to the marginal  likelihood, one can appeal to PAC-Bayes bounds \citep{mcallester_simplified_2003}. The following is presented in \citet[Corollary 4]{germain2016pac}: if the empirical risk $\ell_n(\theta)$ is subgaussian over $\mathcal{D}$ with variance factor $s^2$, then for any probability density $\rho$ of ``trained'' model parameters and any density $\pi$ of ``initialized'' model parameters, with probability at least $1 - \delta$, 
\begin{equation}
\label{eq:PACBayes}
\mathbb{E}_{\theta \sim \rho} L(f_\theta) \leq \mathbb{E}_{\theta \sim \rho} \ell_n(\theta) + \frac{\mathrm{KL}(\rho \Vert \pi) - \log \delta + \tfrac12 s^2}{\sqrt{n}},
\end{equation}
where $\mathrm{KL}(\cdot \Vert \cdot)$ is the Kullback--Leibler divergence. In the interpolating regime where $\rho$ has support on the set of interpolators and $\mathbb{E}_{\theta \sim \rho} \ell_n(\theta) = 0$, the density $\rho$ that minimizes the right-hand side of (\ref{eq:PACBayes}) is the limit of the Bayesian posterior $\rho^\ast(\theta) \propto p(Y \vert X, \theta) \pi(\theta)$ as $\gamma \to 0^+$. For this choice, 
\[
\mathrm{KL}(\rho^\ast \Vert \pi) = -\log Z_n \eqqcolon \mathcal{F}_n,
\]
called the \emph{Bayes free energy}, and so with probability at least $1 - \delta$,
\[
\mathbb{E}_{\theta \sim \rho^\ast} L(f_\theta) \leq \frac{\mathcal{F}_n - \log \delta + \frac12 s^2}{\sqrt{n}}.
\]
Consequently, any approximation of the negative log-marginal likelihood (including the soon-to-be-introduced IIC), directly controls an upper bound on the population risk \citep{hodgkinson_pac-bayesian_2023}. 



\section{IIC for Linear Interpolators}
\label{sec:IICLinear}

Modern overparameterized models are capable of finding a solution with zero error on all training samples; this is \emph{interpolation}. In our setting, an \emph{interpolator} is a predictor $f$ such that for input-output pairs $(x_i,y_i)$, $f(x_i) = y_i$ for $i=1,...,n$.  In what follows, we let $\mathcal M$ denote the collection of all interpolators
\[
\mathcal M := \{\theta\in \Theta:f(x_i,\theta)=y_i\text{ for all }i=1,...,n\}.
\]
In practice, regularizers are a standard tool to mitigate overfitting, and provide a useful selection mechanism for a high-quality interpolator. An interpolator with a regularizer $R(\theta)$ can be defined as a solution to
\begin{equation}
    \theta \in \argmin{\theta \in \Theta} R(\theta)\quad \text{subject to}\quad \theta\in\mathcal M.
    \label{INT}
\end{equation}
In the following, we will take $R(\theta) = \|\theta\|_p^p$. 

\begin{definition}[IIC \citep{hodgkinson_interpolating_2023}]
\label{def:IIC}
    The interpolating information criterion is
    \[
    \mathrm{IIC} \simeq \frac{2}{n} \inf_{\tau>0} \bar{\mathcal F}_{n,\tau},
    \]
    where $\bar{\mathcal F}_{n,\tau}$ is the first-order approximation in $\gamma$ and $\tau$ of the free energy $\mathcal F_{n,\tau} = -\log Z_n$. 
\end{definition}
In this definition, we use $\simeq$ to denote equivalence after discarding any additive term that is constant with respect to $n$ and $d$. This removal focuses attention on the components that are directly relevant to model selection.

For this section, we aim to compute the IIC for linear interpolators and highlight the factors that determine its behaviour. It is important to note that our results appear in the unified form
\begin{equation}
\label{eq:IIC}
\mathrm{IIC}=\reg{A\log R(\theta^\ast)}+\vol{B}
\end{equation}
where $R(\theta)$ is the regularizer that selects the interpolating solution and $\theta^\ast$ is the unique minimizer of $R(\theta)$ subject to $X\theta=Y$. We define $A\log R(\theta^\ast)$ as the \tbox{regcol!20}{regularization} contribution, since it depends on the particular interpolating solution induced by $R$. We define $B$ as the \tbox{volcol!20}{sharpness} contribution. Following \citet{neyshabur_exploring_2017}, sharpness measures how much the empirical loss can increase under small, adversarial perturbations in parameter space, so it captures the local sensitivity of the solution. In our linear setting, this local sensitivity is governed by the design geometry and is encoded by quantities such as $XX^\top$. Each theorem in the following subsections specifies the corresponding pair $(A,B)$ for a given choice of $p$.

Computing the IIC involves estimating the Bayes free energy $\mathcal F_{n,\tau}$, which requires an approximation of the marginal likelihood of the overparameterized model. To avoid the degeneracy of second-order approximations in the overparameterized setting, we rely on \emph{Bayesian duality} (Lemma \ref{lemma:bayesian_duality}; a special case of \citet[Proposition 1]{hodgkinson_interpolating_2023}), which provides a way to replace the original marginal likelihood by a dual quantity that is more tractable.

\begin{lemma}[Bayesian Duality]
\label{lemma:duality_asymptotic}
There exists a dual model with a dual prior $\pi^\ast$ and a dual likelihood $p^\ast(Z) = c_{n,\gamma} e^{-\frac{1}{\gamma}\sum_{i=1}^n \ell(z_i,y_i)}$ whose marginal likelihood $Z_n^\ast$ satisfies $Z_n = Z_n^\ast$. Consequently, as $\gamma \to 0^+$, $Z_n \to \pi^\ast(Y)$. 
\end{lemma}
This result is central to our analysis. It transforms the problem of calculating IIC by evaluating a complex, high-dimensional integral for $Z_n$ into the problem of analyzing the dual prior $\pi^\ast(Y)$. Applying Lemma \ref{lemma:duality_asymptotic} requires taking $\gamma \to 0^+$, putting infinitely large weight on the loss function inside the likelihood, and concentrating the integral around the set of interpolators $\mathcal{M}$. If $\mathcal{H}$ is Hausdorff measure, the dual prior is given by
\begin{align}
\label{eq:target_integral}
\pi^*(Y) &= \frac{1}{\sqrt{\det (XX^\top)}}\int_{X\theta=Y}\pi(\theta)\dd \mathcal H(\theta)\\
\label{eq:target_integral_with_Q}
&= \frac{1}{\sqrt{\det (XX^\top)}}\cdot \int_{\mathbb R^{d-n}} e^{-\frac{1}{\tau}\|X^+Y + Qw\|_p^p}\dd w,
\end{align}
where $Q \in \mathbb R^{d\times(d-n)}$ is a semi-orthogonal matrix such that $\ker (X)=\mathrm {Range} (Q)$, mapping $\mathbb{R}^{d-n}$ into the interpolating set $\mathcal{M}$. The rest of our analysis focuses on estimating this quantity as $\tau \to 0^+$ to compute the IIC. To do so, we require the following assumption. 
\begin{assumption}
\label{assum:unique}
Assume the design matrix $X\in\mathbb R^{n\times d}$ has full row rank $n$ where $d\ge n$. Additionally, assume that there is a unique solution $\theta^\ast$ that solves the minimization problem $\arg\min_{\theta\in\mathbb R^d} \|\theta\|_p \text{ subject to }X\theta = Y$.
\end{assumption}
The full row rank assumption is equivalent to the statement that $XX^\top$ is positive definite and its determinant is strictly positive. 

\subsection{IIC with Ridge Regularization}
First, we mention the simplest case where $p = 2$. This scenario was already considered in \citet{hodgkinson_interpolating_2023}, where the was given by Theorem \ref{thm:pi_p2} below.
\begin{theorem}[\citet{hodgkinson_interpolating_2023}]
\label{thm:pi_p2}
If $p=2$,
\[
\mathrm{IIC} = \reg{\log \|\theta^\ast\|^2_2} +\vol{\frac{1}{n}\log\det(XX^\top) -\log n}.
\]
\end{theorem} 
Immediately, the IIC can be decomposed into a regularization term and a sharpness term. The term $\log\|\theta^\ast\|_2^2$ consists of the regularization strength by the optimal parameter estimator. Analogous quantities commonly appear as the ``bias'' term in asymptotics of mean-squared error for $\ell^2$ linear interpolators \citep{hastie2022surprises}. For linear models, the second term is a rescaling of
\[
\log \det (XX^\top) = \log \textstyle\det_+ I_{\theta^\ast},
\]
where $I_{\theta^\ast}$ is the empirical Fisher information matrix, and $\det_+$ is the pseudodeterminant, or the product of all nonzero eigenvalues. This term clearly describes the sharpness in the loss landscape, measuring effects of local perturbations. These two terms are evidently in tension: since $\theta^\ast = X^+ Y$ where $X^+$ is the Moore--Penrose pseudoinverse of $X$, decreasing the scale of $X$ will increase the regularization term, but decrease the sharpness term. 

\subsection{IIC with Smooth Regularization}
Our next case considers $p \geq 2$, which requires more work than the $p = 2$ case, but the Laplace approximation machinery presented in \citet{hodgkinson_interpolating_2023} still applies. The following Theorem \ref{thm:pi_p_ge_2} is our first novel result. 
\begin{theorem}
\label{thm:pi_p_ge_2}
If $p\ge 2$ and $d\le \frac{np}{p-2}$, 
\begin{align*}
\mathrm{IIC} = \mathbox{regcol!20}{\;\frac{2d-p(d-n)}{np} \log \|\theta^\ast\|_p^p\;} 
+ \mathbox{volcol!20}{\;\frac{1}{n}\log \det(XX^\top)+\frac{p-2}{n}\sum_{j=1}^d\log|\theta^\ast_j|+ K_1(p,d,n)\;},
\end{align*}
where $K_1(p,d,n)$ is presented in (\ref{eq:K1Const}) in Appendix \ref{sec:pf_IIC_pge2}. 
\end{theorem}
The function $K_1(p,d,n)$ only consists of the dimension parameters and represents an ambient volume term. While the same decomposition holds as in the $p = 2$ case, we observe that there is a peculiar upper bound on the model dimension. This is to ensure that the minimum point in Definition~\ref{def:IIC} exists. If $p > 2$ and $d$ is sufficiently large, concentrating mass around the interpolator point estimate maximizes the marginal likelihood, analogous to the \emph{cold posterior effect} \citep{wenzel2020good}. Since there is no appropriate sense of intrinsic uncertainty in this setting, we consider this to be a case where the marginal likelihood approach fails, and a more direct approach may be necessary. The sharpness term inherits the same log-determinant from the $p = 2$ case with an additional term $\frac{p-2}{n} \sum_{j=1}^d\log |\theta_j^*|$ involving the model parameters as well. This extra factor arises from the curvature of the $\ell^p$ prior at the interpolator and encodes how mass is distributed along different coordinate directions. As $p$ grows, the regularizer puts more weight on large coefficients, which makes the regularization term decrease faster and amplifies the influence of the parameter-sharpness term in the overall tradeoff.

\subsection{IIC with Sparse Regularization}
The final case considers $p = 1$, invoking a regularizer that induces sparsity. This case is quite different in nature from the $p \geq 2$ setting, as the machinery in \citet{hodgkinson_interpolating_2023} no longer applies. Our strategy here takes advantage of equivalent characterizations of uniqueness (guaranteed in Assumption \ref{assum:unique}) coming from basis pursuit theory \citep{zhang_necessary_2015}. 
\begin{theorem}
\label{thm:iic_peq1}
The IIC for $\ell^1$ interpolators ($p=1$) satisfies
\begin{align*}
\mathrm{IIC}= \reg{2\log \|\theta^\ast\|_1 } 
+ \vol{\frac{1}{n}\log\det(XX^\top)-\frac{2}{n}\log (V_0)+K_2(d,n)},
\end{align*}
where $V_0 = \int_{\ker X} \mathds 1\left\{\|z'_C\|_1 + s\cdot z'_S\le 1\right\}\dd z'$ is finite, $S = \mathrm{supp}(\theta^\ast)$, $C = \{1,\dots,d\} \setminus S$, and $K_2(d,n)$ is presented in (\ref{eq:K2}) in Appendix \ref{sec:pf_IIC_p1}. 

Furthermore, if $|S| = n$, then
\[
V_0 = \frac{2^{d-n}\sqrt{\det(I+\Psi^{\top}\Psi)}}{(d-n)!}\cdot\prod_{k=1}^{d-n}\frac{1}{1-\psi_{k}^{2}},
\]
where $\Psi = X_S^{-1} X_C$, $X_S$ is the matrix $X$ keeping only the column indices in $S$ and likewise for $X_C$, and $\psi~=~\Psi^\top \mathrm{sign}(\theta^\ast)$. 
\end{theorem}
In the case where $n = 1$, it is always true that $|S| = n = 1$, and the formula for $V_0$ reduces to a convenient form. 
\begin{corollary}
    \label{thm:iic_p1_n1}
    If $p=1$ and $n=1$, $I=\argmax{j}|x_j|$,
    {\small\begin{align*}
    \mathrm{IIC} = \reg{\vphantom{\prod_{j}^d}2\log\|\theta^\ast\|_1} -\vol{2\log \left(\frac{1}{2\|x\|_\infty}\prod_{j=1, j\neq I}^d \frac{x_I^2}{x_I^2-x_j^2}\right)}
    \end{align*}}
\end{corollary}
 
The regularization term $2\log \|\theta\|_1$ penalizes interpolators with a large $\ell^1$ norm, \emph{a la} LASSO. The sharpness term contains the same contribution $(1/n) \log \det(XX^\top)$ and an ambient term $K_2(d, n)$, but also depends on a volume $V_0$ that is unique to the sparse case. 
Intuitively, $V_0$ measures the extent of the space in  $\mathrm{ker}(X)$ that is within unit $\ell^1$ distance in the directions compatible with the support and signs of $\theta^\ast$. As we see in the $n = 1$ case in Corollary \ref{thm:iic_p1_n1}, this quantity can increase quickly in the model size. The result is that unlike the $p \geq 2$ cases, the sparsity inducing regularizer can break the regularization--sharpness tradeoff by ensuring that \textbf{both terms decrease in model size}. The result is that the $\ell^1$ linear interpolators may achieve the best empirical performance, the smallest IIC, and the largest marginal likelihood among the regularizers considered here. 

Taken together, these three cases show that, for linear interpolators with $\ell^p$ priors, the IIC always decomposes into a regularization term that depends only on the norm of the chosen interpolator and a sharpness term that depends on the data geometry and, in nonquadratic cases, on local curvature information. 


\section{Numerical Experiments}
\label{sec:Experiments}
We evaluate IIC in overparameterized linear regression under feature constructions that allow us to vary the dimension $d$. Across experiments, we consider $\ell^p$ regularization and report the IIC objective together with its components: a regularization term and a sharpness term. We mainly focus on the post-interpolation regime where training error is (near) zero, and generalization is measured by test mean squared error (MSE). Our examinations begin with a real-world ``Facial and Oral Temperature Dataset" (FLIR) from PhysioNet~\citep{wang2023physionet, Wang_Infrared_2022, physionet} as our main case study, and then extend the evaluation of correlation across additional datasets. 

\paragraph{Visualizing IIC and the Tradeoff.} 
\begin{figure}
    \centering
    \begin{subfigure}{0.24\linewidth}
    \includegraphics[width=\linewidth]{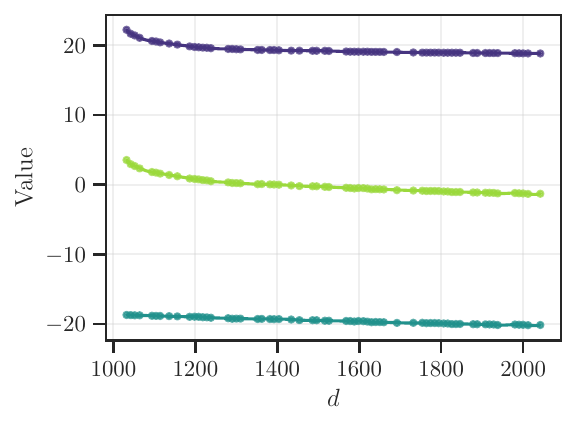}\vspace{-2mm}
    \caption{$p=1$}
    \label{fig:iic_p1}
    \end{subfigure}\hspace{0.01\linewidth}
    \begin{subfigure}{0.24\linewidth}
    \includegraphics[width=\linewidth]{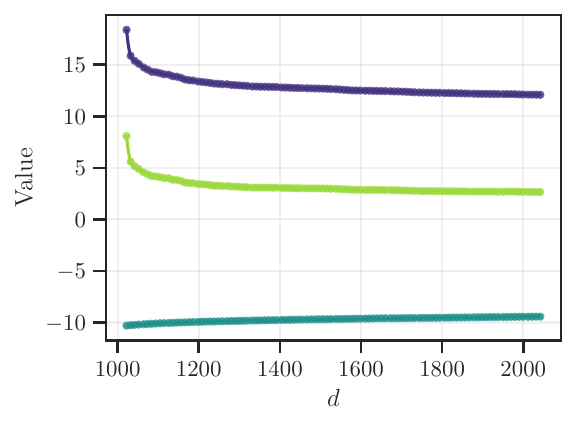}\vspace{-2mm}
    \caption{$p=2$}
    \label{fig:iic_p2}
    \end{subfigure}
    \begin{subfigure}{0.24\linewidth}
    \includegraphics[width=\linewidth]{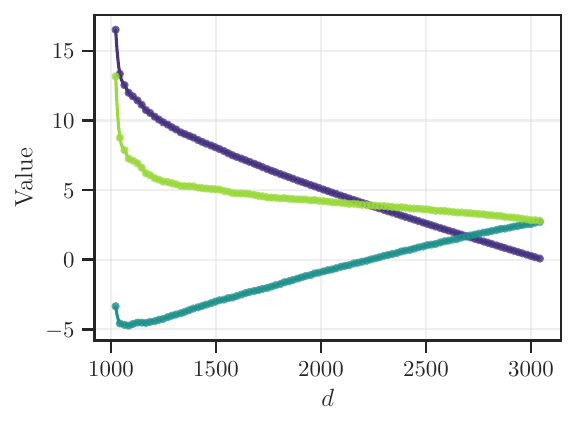}\vspace{-2mm}
    \caption{$p=3$}
    \label{fig:iic_p3}
    \end{subfigure}\hspace{0.01\linewidth}
    \begin{subfigure}{0.24\linewidth}
    \includegraphics[width=\linewidth]{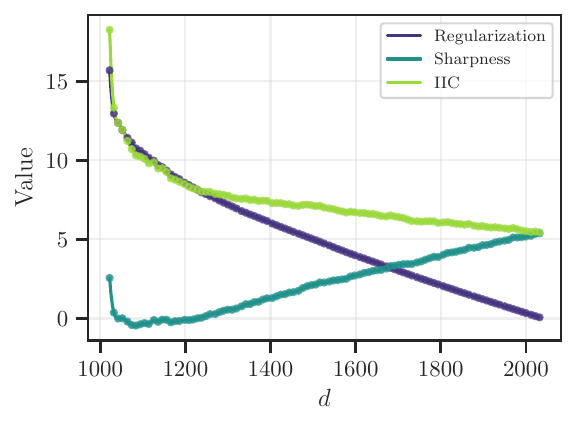}\vspace{-2mm}
    \caption{$p=4$}
    \label{fig:iic_p4}
    \end{subfigure}
    \caption{Decomposition of the Interpolating Information Criterion (green) for minimum $\ell^p$-norm interpolating solutions (with varying $p$) using random Fourier features as a tradeoff between the effect of regularization (purple) and local sharpness (blue).\vspace{-.5cm}}
    \label{fig:IIC_RFF}
\end{figure}
We consider the FLIR dataset augmented with random Fourier features (RFF) \citep{rahimi_random_2007}, constructed following the experimental setup in \citet{belkin_reconciling_2019}. For each original input $x_i\in\mathbb{R}^{d_0}$, RFF maps it into a higher-dimensional feature vector $d$. We sample $d_h=\lfloor d/2\rfloor$ vectors $v_k\sim N(0,\sigma^{-2}I_{d_0}),\ k=1,...,d_h$, and construct the mapped features as $\frac{1}{\sqrt{d_h}}\left (\cos(v_k^\top x_i),\sin(v_k^\top x_i)\right )$. Our initial investigations study how IIC behaves as the feature dimension $d$ increases beyond the interpolation threshold (Figure \ref{fig:IIC_RFF}). We focus on the post-interpolation regime, where training error is (near) zero, such that the variation across $d$ reflects how the selected interpolating solution changes rather than improvements in the training fit. As the feature dimension increases within the overparameterized regime, the overall IIC decreases steadily. In the decomposition, the regularization term decreases with $d$, indicating that on higher-dimensional interpolation manifolds, one can find lower-norm interpolators. In contrast, the sharpness term increases as $d$ grows, indicating an increasing local geometric contribution around the selected interpolator. These two terms move in opposite directions and together \textbf{form a tradeoff}.

For $p=1$, we plot the decomposition only when the support has the same cardinality as the data set (Figure \ref{fig:iic_p1}). This result is noteworthy for demonstrating the change in the trend from the other $\ell^p$ cases. Comparing with the larger $p$ in Figure \ref{fig:IIC_RFF}, we can find that the decreasing speed of regularization slows down as $p$ decreases, and the sharpness term also changes from monotone increasing to decreasing. This effect is even more pronounced in the special case $p=1,\ n=1$, where the sharpness term decreases more noticeably (Figure \ref{fig:IIC_p1_n1}). We find that by decreasing $p$ and inducing sparsity, one can trade the effect of the regularization term to significantly decrease the sharpness term. This demonstrates one possible reason why the $\ell^1$ penalty is commonly preferred over higher orders.

To check that the decrease in IIC is not an artifact of increasing dimension, we repeat the analysis on the same FLIR dataset using polynomial features as a contrasting example (Figure \ref{fig:iic_polynomial}). Prior work \citep{schaeffer_double_2023, hodgkinson_interpolating_2023} has noted that polynomial feature regression need not display a second descent in test error as model size increases. In this setting, IIC does not continue to decrease as before. Instead, it varies with dimension in a way that tracks the observed test MSE, showing a qualitatively different trend from the RFF case. This contrast provides evidence that IIC is sensitive to whether a model class exhibits benign overfitting, and that its decomposition reflects meaningful structure in the interpolating regime.

\begin{figure}
    \centering
\begin{subfigure}{0.35\linewidth}
    \includegraphics[width=1\linewidth]{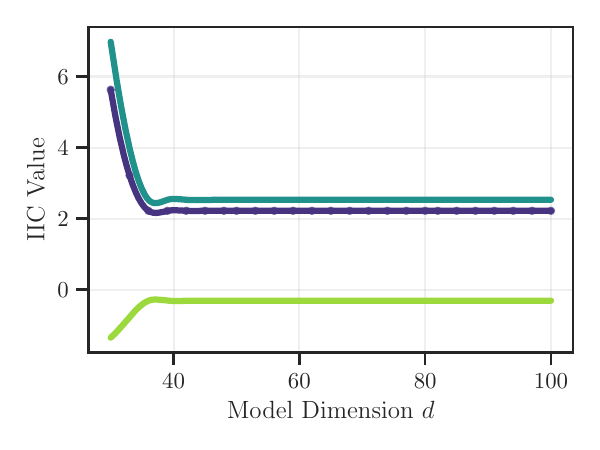}
    \caption{$p=2$}
    \label{fig:iic_polynomial_p2}
\end{subfigure}
\begin{subfigure}{0.35\linewidth}
    \includegraphics[width=1\linewidth]{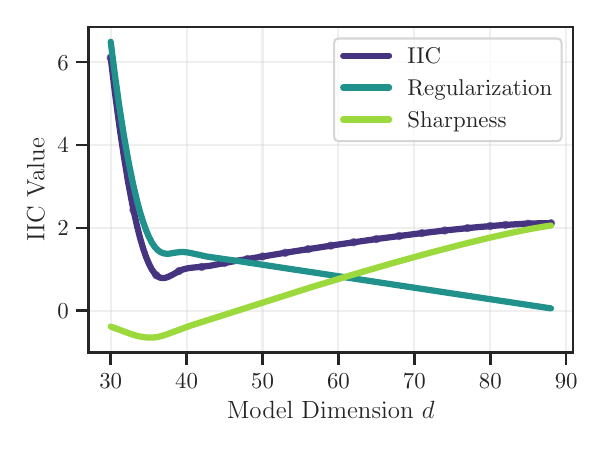}
    \caption{$p=3$}
    \label{fig:iic_polynomial_p3}
\end{subfigure}
\caption{Decomposition of the Interpolating Information Criterion (green) for minimum $\ell^p$-norm interpolating solutions (with varying $p$) using \textbf{polynomial features} as a tradeoff between the effect of regularization (purple) and local sharpness (blue). These estimators \emph{perform poorly}.\vspace{-.3cm}}
\label{fig:iic_polynomial}
\end{figure}

\paragraph{Correlations with Generalization and Model Size.} We next test whether IIC provides a reliable proxy for generalization across datasets, and the trend of the tradeoff terms when the dimension of features grows. We use eight datasets: FLIR, CIFAR-10, MNIST, Fashion-MNIST, CIFAR-100, SVHN, scikit-learn digits, and USPS digits. For each dataset, we construct RFF features and vary the feature dimension $d$ to traverse the interpolating regime. For each configuration we compute the IIC objective and evaluate generalization by test MSE. We then compute the Spearman correlation $\rho$ between IIC and generalization error, regularization term and dimension, sharpness term and dimension (Table \ref{tab:spearman_corr}). The first correlation assesses whether IIC provides the same ordering as generalization error, while the latter two diagnose the direction of the interpolating tradeoff captured by the decomposition.

Overall, IIC is strongly rank-consistent with test MSE in most settings, and the decomposition exhibits a clear pattern in $d$: the regularization term is typically negatively correlated with dimension (larger $d$ admits smaller-norm interpolators), whereas the sharpness term is typically positively correlated with dimension (local geometric cost increases with $d$). Together, these correlations support the view that IIC tracks generalization in the interpolating regime through an explicit regularization–sharpness tradeoff.

\begin{table}[t]
\centering
\caption{Spearman Correlations (with 95\% confidence intervals) for IIC vs. test mean-squared error; the regularization term with respect to model size; and the sharpness term with respect to model size.\vspace{.3cm}}
\label{tab:spearman_corr}
\begin{adjustbox}{width=0.65\textwidth,center}
\begin{tabular}{@{} l c c c c @{}}
\toprule
\textbf{Dataset} & \textbf{Reg} & \textbf{IIC vs. test MSE} & \textbf{reg vs. dim} & \textbf{sharpness vs. dim} \\
\midrule
\textbf{FLIR} & $p=1$ & $0.847\pm0.030$ & $-0.999\pm0.001$ & $-0.998\pm0.001$ \\
 & $p=2$ & $0.799\pm0.213$ & $-1.000\pm 0.000$ & $1.000\pm0.000$ \\
 & $p=3$ & $0.785\pm 0.214$  & $-1.000\pm 0.000$ & $0.973\pm0.012$ \\
\midrule
\textbf{CIFAR-10} & $p=1$ & $0.997\pm0.001$ & $-0.984\pm0.008$ & $-0.999\pm0.001$ \\
 & $p=2$ & $1.000\pm0.000$ & $-1.000\pm0.000$ & $1.000\pm0.000$ \\
 & $p=3$ & $1.000\pm0.000$ & $-1.000\pm0.000$ & $0.929\pm0.015$ \\
\midrule
\textbf{MNIST} & $p=1$ & $0.999\pm0.002$ & $-0.992\pm0.006$ & $-0.999\pm0.001$ \\
 & $p=2$ & $1.000\pm0.000$ & $-1.000\pm0.000$ & $1.000\pm0.000$ \\
 & $p=3$ & $1.000\pm0.000$ & $-1.000\pm0.000$ & $0.894\pm0.057$ \\
\midrule
\textbf{FMNIST} & $p=1$ & $0.997\pm0.002$ & $-0.992\pm0.002$ & $-0.998\pm0.002$ \\
 & $p=2$ & $1.000\pm0.000$ & $-1.000\pm0.000$ & $1.000\pm0.000$ \\
 & $p=3$ & $1.000\pm0.000$ & $-1.000\pm0.000$ & $0.896\pm 0.036$ \\
\midrule
\textbf{CIFAR-100} & $p=1$ & $0.998\pm0.003$ & $-0.967\pm0.034$ & $-1.000\pm0.000$ \\
 & $p=2$ & $1.000\pm0.000$ & $-1.000\pm0.000$ & $1.000\pm0.000$ \\
 & $p=3$ & $1.000\pm0.000$ & $-1.000\pm0.000$ & $0.947\pm0.020$ \\
\midrule
\textbf{SVHN} & $p=1$ & $0.996\pm0.004$ & $-0.979\pm0.025$ & $-0.999\pm0.001$ \\
 & $p=2$ & $1.000\pm0.000$ & $-1.000\pm0.000$ & $1.000\pm0.000$ \\
 & $p=3$ & $1.000\pm0.000$ & $-1.000\pm0.000$ & $0.938\pm0.040$ \\
\midrule
\textbf{Scikit-learn digits} & $p=1$ & $0.936\pm0.079$ & $-0.994\pm0.010$ & $-0.998\pm0.001$ \\
 & $p=2$ & $0.984\pm0.011$ & $-1.000\pm0.000$ & $1.000\pm0.000$ \\
 & $p=3$ & $0.987\pm0.010$ & $-1.000\pm0.000$ & $0.925\pm0.021$ \\
\midrule
\textbf{USPS digits} & $p=1$ & $0.998\pm0.002$ & $-0.992\pm0.015$ & $-0.998\pm0.001$ \\
 & $p=2$ & $1.000\pm0.002$ & $-1.000\pm0.000$ & $1.000\pm0.000$ \\
 & $p=3$ & $1.000\pm0.000$ & $-1.000\pm0.000$ & $0.929\pm0.013$ \\
\bottomrule
\end{tabular}
\end{adjustbox}
\end{table}

\section{Conclusion}

This work studies model selection in overparameterized settings, where exact interpolation of the training data is both feasible and common. The interpolating regime is of particular interest in deep learning due to large neural network models that exhibit very small training losses. Although we work with linear models, our findings shed insights on how to quantify model performance when the training error is no longer a useful measurement. 

As summarized in Table~\ref{tab:tradeoffs-summary}, different statistical criteria for model selection are underpinned by distinct theoretical tradeoffs. Classical approaches, including MSE and BIC, are only effective in underparameterized regimes and become ineffective when models are large enough to perfectly interpolate the training data. In contrast, the interpolating information criterion introduces a regularization-sharpness tradeoff tailored for this modern regime. The first term in each IIC expression consistently acts as a regularization term, related to the norm of the estimated parameters and reflecting the degree to which the final solution aligns with the prior induced by the regularizer. The second component of the IIC is the sharpness term, which captures the geometric properties of the solution space and has a close link to the Fisher information. 

Another significant result in this work is the formalization of explicit expressions for the IIC in linear interpolators. We first provide a general criterion for models using an $\ell^p$ penalty when $p\ge 2$. Crucially, we then extend this analysis to the LASSO interpolator, which uses an $\ell^1$ regularizer to induce sparsity. This extension reveals a qualitative difference between smooth and sparse regularization. Under smooth regularization, the regularization term decreases while the sharpness term increases with model size, whereas under sparse regularization both terms decrease. As a result, the IIC under sparse regularization drops more clearly as the model size grows, indicating better model performance. The sparse regularization result is particularly informative, as it sets a direct link between the regularization-sharpness principle and the widely used practice of sparse modeling.

\paragraph{Limitations.}
Our work provides initial support for the existence of a novel tradeoff, but is far from conclusive. The IIC framework can hold for more general models, and while an extension of our findings (with the $\ell^1$ regularizer in particular) to that level of generality would be very challenging, it would be substantially more relevant to the deep learning context. Furthermore, for $p\ge 2$, there is an upper bound for the IIC to ensure the existence of an optimal temperature $\tau$. Consequently, the IIC framework is not effective for studying interpolators with large $d$ when $p > 2$; a new strategy is required here.


%


\bibliography{reference}
\bibliographystyle{icml2025}

\newpage
\appendix
\onecolumn

\section{Laplace Approximation}\label{sec:Laplace_approximation}
The Laplace approximation provides a simple asymptotic estimate for integrals whose main contribution comes from a small neighborhood of an optimizer. It does this by using a quadratic expansion of the exponent at that point, turning the integral into a Gaussian-type expression. We apply it here as a technical step to obtain an explicit approximation that makes the regularization and local curvature effects transparent in the subsequent derivations. The following is derived from \citet[Theorem 15.2.2]{simon_advanced_2015}.
\begin{theorem}[Laplace Approximation]
\label{thm:Laplace}
Let $g,h\colon \mathbb R^n\to \mathbb R$ satisfy the following conditions: 
\begin{itemize}
\item $h(x)$ has a unique global minimum at $x^* \in \mathbb R^n$;
\item $h(x)$ is $C^2$ differentiable near $x^*$, and its Hessian matrix at this point $\nabla^2 h(x^*)$ is strictly positive definite and symmetric;
\item for some $r>0$ and $h_r$, for all $\|x-x^*\|>r$, $h(x)\geq h_r>h(x^*)$; and
\item $g(x)$ is continuous, integrable and strictly positive at $x^*$.
\end{itemize}
Then as $\gamma \to 0^+$, the following asymptotic relation holds:
\[
S_n := \int_{\mathbb R^n}g(x)e^{-\frac{1}{\gamma}h(x)}\dd x\sim (2\pi\gamma)^{\frac{n}{2}}g(x^*)e^{-\frac{1}{\gamma}h(x^*)}\det (\nabla^2 h(x^*))^{-\frac{1}{2}}.
\]
\end{theorem}
Here, $\sim$ signifies asymptotic equivalence, indicating that the ratio of the two functions on either side tends to one (as $\gamma \to 0^+$).

\section{Bayesian Duality} \label{sec:Bayes_duality}
Direct Laplace expansions in parameter space fail when $d>n$ because the local Hessian of the log likelihood is singular and the quadratic approximation is invalid. To avoid this singularity, we adopt a strategy based on a \emph{Bayesian duality} principle, following ~\citet{hodgkinson_interpolating_2023}. This principle establishes an equivalence between the marginal likelihood of the original overparameterized model and that of a dual, underparameterized model. 
\begin{lemma}[Proposition 1 of~\citet{hodgkinson_interpolating_2023}]
    \label{lemma:bayesian_duality}
    For a likelihood $p_{n,\gamma}(y\mid x,\theta)$ and prior $\pi(\theta)$, there is a dual likelihood $p^\ast _{n,\gamma}(y\mid z) = c_{n,\gamma}(z)e^{-\frac{1}{\gamma}L(z,y)}$ and a dual prior $\pi^\ast(z) = \int_{F^{-1}(z)}\frac{\pi(\theta)}{\det J(\theta)^{1/2}}\dd \mathcal{H}^{d-mn}(\theta)$ such that
    \begin{equation}
    \int_{\mathbb R^d}p_{n,\gamma}(y|x,\theta)\pi(\theta)\dd \theta = \int_{\mathbb R^{n}}p^*_{n,\gamma}(y|z)\pi^*(z)\dd z.
\end{equation}
\end{lemma}
 The left integral is the original marginal likelihood $Z_n$ and the right integral is the dual representation $Z^*_n$. In our linear regression setting we set $\gamma = 2\sigma^2$ and define $L(Y,\hat Y) = \sum_{i=1}^n\ell(y_i,\hat y_i)$ with squared error loss. The likelihood is 
\begin{align*}
p(Y|X,\theta) &= (\pi \gamma)^{-\frac{n}{2}}\exp\left (-\frac{1}{\gamma}\sum_{i=1}^n (y_i-x_i^\top \theta)^2\right)\\
&= (\pi \gamma)^{-\frac{n}{2}}\exp\left (-\frac{1}{\gamma}L(Y,X\theta)\right).
\end{align*}
The corresponding dual likelihood is
\[
p^*(Y|z) = (\pi \gamma)^{-\frac{n}{2}}\exp\left (-\frac{1}{\gamma}L(Y,z)\right).
\]
Let $F(\theta) := (f(x_i,\theta))_{i=1}^n = X\theta$. The dual prior is defined by
\[
\pi^\star(z) = \int_{F^{-1}(z)}\frac{\pi(\theta)}{\det J(\theta)^{\frac{1}{2}}}\dd \mathcal H^{d-n}(\theta)
\]
where the integral is taken over the level set $X\theta = z$ and $\mathcal{H}^\alpha$ is Hausdorff measure. Then the Jacobian of $F(\theta)$ is
\[DF(\theta) = \frac{\partial F(\theta)}{\partial\theta} = X\]
and $J(\theta) := DF(\theta)DF(\theta)^\top = XX^\top$, which is constant with respect to $\theta$.

\section{Proofs of the results in the main text}
\subsection{Proof of Theorem \ref{thm:pi_p2}: IIC with $p=2$}\label{sec:pf_IIC_p2}
We first get the approximation of the dual prior, and use Lemma \ref{lemma:duality_asymptotic} and Definition~\ref{def:IIC}  to get the result of this theorem. 
\begin{lemma}
    For linear regression with $\ell^2$ regularizer, the dual prior $\pi^*(Y)$ is approximated as $\tau \to 0^+$ by
    \[
    \pi^\star(Y) \sim (\pi\tau)^{-\frac{n}{2}}(\det XX^\top)^{-\frac{1}{2}}\exp\left(-\frac{1}{\tau}\|X^+Y\|_2^2\right).
    \]
\end{lemma}
\begin{proof}
Using Laplace's method from Theorem \ref{thm:Laplace}, we treat $\tau$ as an approximation parameter and define $h(w) = \|X^+Y + Qw\|_2^2,\ g(w) = 1$. To find local minima of $h(w)$, we compute its gradient
\[
\nabla h(w) = 2Q^\top(X^+Y + Qw) = 2Q^\top X^+Y + 2w =0.
\]
Setting the gradient to zero yields the minimum $w_0 = -Q^\top X^+Y$. The Hessian matrix of $h(w)$ is
\[
\nabla^2 h(w) = 2I_{d-n}.
\]
This matrix is positive definite, confirming that $w_0$ is indeed a global minimum. At this point, the function $h(w)$ evaluates to
\[
h(w_0) = \|X^+Y - QQ^\top X^+Y\|_2^2 = \|X^+Y\|_2^2.
\]
The last equality holds because $QQ^\top X^+Y = Q(XQ)^\top(XX^\top)^{-1}Y=0$. The determinant of the Hessian at $w_0$ is
\[
\det \nabla^2 h(w_0) = 2^{d-n}.
\]
Thus the integral is approximated as
\begin{align*}
    \pi^\star(Y) &\sim (\pi\tau)^{-\frac{d}{2}}(\det XX^\top)^{-\frac{1}{2}}(2\pi\tau)^{\frac{d-n}{2}}2^{-\frac{d-n}{2}}\exp\left(-\frac{1}{\tau}\|X^+Y\|_2^2\right)\\
    &\sim (\pi\tau)^{-\frac{n}{2}}(\det XX^\top)^{-\frac{1}{2}}\exp\left(-\frac{1}{\tau}\|X^+Y\|_2^2\right) \quad \text{as } \tau \to 0^+,
\end{align*}
and the result follows.
\end{proof}

Building on the framework established in Definition \ref{def:IIC}, we begin by determining the Bayes free energy, which is defined as the negative logarithm of the marginal likelihood. By Lemma \ref{lemma:duality_asymptotic} and Theorem \ref{thm:pi_p2}, when $\gamma\to 0^+$ and then $\tau\to 0^+$, 
\[
\bar{\mathcal{F}}_{n,\tau} = \frac{n}{2}\log (\pi \tau)+\frac{1}{2}\log\det(XX^\top)+\frac{1}{\tau}\|X^+Y\|_2^2.
\]
The IIC is then defined as the scaled infimum of this free energy with respect to
\[\text{IIC} = \frac{2}{n}\inf_\tau\bar{\mathcal F}_{n,\tau}.\]
To find the value of $\tau$ that minimizes $\bar{\mathcal{F}}_{n,\tau}$, we compute the derivative with respect to $\tau$ and set it to 0
\[
\frac{\partial \bar{\mathcal{F}}_{n,\tau}}{\partial \tau} = \frac{n}{2\tau} - \frac{1}{\tau^2}\|X^+Y\|_2^2 = 0.
\]
Solving for $\tau$ gives the optimal value $\tau^* = \frac{2}{n}\|X^+Y\|^2_2$. Substituting the optimizer $\tau^*$ back into Definition \ref{def:IIC} and discarding any additive term that is constant with respect to $n$ and $d$ yields 
\begin{equation}
\label{eq:IIC_p_2}
\text{IIC} = \underset{\text{regularization}}{\underbrace{\log \|X^+Y\|^2_2}} +\underset{\text{sharpness}}{\underbrace{\frac{1}{n}\log\det(XX^\top) -\log n}}.
\end{equation}
This expression makes the regularization term and the sharpness term explicit.

\subsection{Proof of Theorem \ref{thm:pi_p_ge_2}: IIC with $p\ge2$}\label{sec:pf_IIC_pge2}
Following the same proof strategy as the previous section, we first get the approximation of the dual before getting the result of this theorem.
\begin{lemma}
For the general case when $p\ge2$, the dual prior $\pi^*(Y)$ is approximated by
\begin{align*}
\pi^*(Y)\sim & \frac{(2\pi\tau)^{\frac{d-n}{2}}\tau^{-\frac{d}{p}}}{(p(p-1))^{\frac{d-n}{2}}2^{d}\Gamma\left(\frac{1}{p}+1\right)^{d}}(\det XX^\top)^{-\frac{1}{2}}e^{-\frac{1}{\tau}\|\theta^\ast\|_p^p}\prod_{j=1}^d|\theta_j^*|^{-\frac{p-2}{2}}
\end{align*}
as $\tau \to 0^+$.
\end{lemma}

\begin{proof}
When $p\ge2$, the normalizing constant for the prior is
\[
c_{p,\tau} = 2^{-d}\Gamma\left(\frac{1}{p}+1\right)^{-d}\tau^{-\frac{d}{p}}
\]
Then the integral of interest given by \ref{eq:target_integral} becomes 
\[
\pi^\star(Y) = c_{p,\tau}(\det XX^\top)^{-\frac{1}{2}}\int_{\mathbb R^{d-n}} \exp\left(-\frac{1}{\tau}\|X^+Y + Qw\|_p^p\right)\dd w.
\]
To apply the Laplace approximation from Theorem \ref{thm:Laplace}, we denote $h(w;p) = \|X^+Y+Qw\|_p^p$ and $g(w)=1$. The integral is approximated by evaluating the function's behavior around its minimum $w_0$. 

First, we compute the Hessian of $h(w)$, which will also confirm the condition of Laplace method. For the integration range, $\theta = X^+Y +Qw$ given by the construction of $Q$. Then the function $h(w;p)$ can be written as $\sum_{j=1}^d |\theta_j|^p$. The first partial derivative with respect to a component $w_u$ is
\[
\frac{\partial h(w;p)}{\partial w_u} = p\sum_{j=1}^d|\theta_j|^{p-1}\cdot \mathrm{sgn}(\theta_j)\frac{\partial \theta_j}{\partial w_u}
\]
where $u=1,...,d-n$. Denote $q_{ju}$ as the $(j,u)$-th element of Q, then the $j$-th element of $\theta$ is $\theta_j = (X^+Y)_j + \sum_{u=1}^{d-n}q_{ju}w_u$. From this 
\[
\frac{\partial h(w;p)}{\partial w_u} = p\sum_{j=1}^d|\theta_j|^{p-1}\cdot \mathrm{sgn}(\theta_j)q_{ju}.
\]
For non-zero $\theta_j$, we can apply the chain rule to obtain the second derivative with respect to $w_u$ and $w_v$
\begin{align*}
    \frac{\partial^2h(w;p)}{\partial w_u\partial w_v}
    &= p(p-1)\sum_{j=1}^d|\theta_j|^{p-2}\frac{\partial \theta_j}{\partial w_v}q_{ju}
    + p\sum_{j=1}^d|\theta_j|^{p-1}\frac{\partial \mathrm{sgn}(\theta_j)}{\partial w_v}q_{ju}\\
    &= p(p-1)\sum_{j=1}^d|\theta_j|^{p-2}q_{ju}q_{jv}.
\end{align*}
The second step follows since the partial derivative of the sign function is zero. For $\theta_j=0$, we can compute the derivative manually. 
Combining all the elements into the matrix form gives 
\[
\nabla^2 h(w;p)=p(p-1)Q^\top\Lambda_p Q.
\]
Here $\Lambda_{p} =  \mathrm{diag}(|\theta_1|^{p-2},...,|\theta_d|^{p-2})$. We define $w^*$ as the minimum point of $h(w;p)$, which is equivalent to $\theta^\ast = \arg\min_{\theta}\|\theta\|_p^p$. $\Lambda^*_p$ is the diagonal matrix corresponding to $\theta^\ast$. Then the Hessian matrix is
\[
\det \nabla^2 h(w^*;p) = p^{d-n}(p-1)^{d-n}\det(Q^\top\Lambda^*_pQ).
\]
Introducing the identity matrix under a limit, this expression can be given by
\[
\det(Q^\top\Lambda^*_pQ) =\lim_{\lambda\to 0^+}\det(\lambda I+Q^\top \Lambda^*_p Q).
\]
To simplify the determinant of a product involving non-square matrices, we apply the Weinstein–Aronszajn identity~\citep{howland_weinstein-aronszajn_1970}. It allows us to exchange $Q$ to make use of $Q^\top Q=I_{d-n}$: 
\begin{align*}
\det(Q^\top\Lambda^*_pQ) &=\lim_{\lambda\to 0^+}\det(\lambda I+Q^\top Q \Lambda^*_p) \\
&= \lim_{\lambda\to 0^+}\det(\lambda I+\Lambda^*_p) \\
&= \prod_{j=1}^d |\theta^\ast_j|^{p-2}.
\end{align*}
So the Hessian at $w^*$ is 
\[
\det \nabla^2 h(w^*;p) = p^{d-n}(p-1)^{d-n} \prod_{j=1}^d |\theta^\ast_j|^{p-2}.
\]

Substituting the results into Theorem \ref{thm:Laplace} gives the approximation of the dual prior.
\end{proof}

Taking the negative logarithm of this result, the Bayes free energy is
\begin{align*}
    \bar{\mathcal F}_{n,\tau} &= \frac{d-n}{2}\log p(p-1) +d\log2\Gamma\left(\frac{1}{p}+1\right)-\frac{d-n}{2}\log 2\pi+\frac{2d-p(d-n)}{2p}\log \tau \\
    &+\frac{1}{2}\log \det XX^\top + \frac{1}{\tau}\|\theta^\ast\|_p^p+ \frac{p-2}{2}\sum_{j=1}^d\log|\theta^\ast_j|.
\end{align*}
To find the optimal $\tau$ that minimizes this energy, we differentiate with respect to $\tau$ and equate the derivative to 0:
\[
\frac{\partial \bar{\mathcal F}_{n,\tau}}{\partial\tau} = \frac{2d-p(d-n)}{2p\tau}-\frac{1}{\tau^2}\|\theta^\ast\|_p^p = 0.
\]
Solving for $\tau$ yields the optimal value
\[
\tau^* = \frac{2p}{2d-p(d-n)}\|\theta^\ast\|_p^p.
\]
Substituting $\tau^*$ back into Definition \ref{def:IIC} gives
\begin{align*}
\text{IIC} &= \inf_\tau \frac{2}{n}\bar{\mathcal F}_{n,\tau}\\
&=\frac{d-n}{n}\log p(p-1) +\frac{2d}{n}\log2\Gamma\left(\frac{1}{p}+1\right)-\frac{d-n}{n}\log 2\pi +\frac{1}{n}\log \det XX^\top\\
    &+ \frac{2d-p(d-n)}{np}\log \frac{2pe}{2d-p(d-n)} + \frac{2d-p(d-n)}{np}\log\|\theta^\ast\|_p^p+ \frac{p-2}{n}\sum_{j=1}^d\log|\theta^\ast_j|.
\end{align*}
To simplify this term, we define the correction function as
\begin{equation}
\label{eq:K1Const}
K_1(p,d,n) = \frac{d-n}{n}\log \frac{p(p-1)}{2\pi} +\frac{2d}{n}\log2\Gamma\left(\frac{1}{p}+1\right) + \frac{2d-p(d-n)}{np}\log \frac{2pe}{2d-p(d-n)}.
\end{equation}
Then the IIC can be decomposed into two primary parts
\[
\text{IIC}=\underset{\text{regularization}}{\underbrace{\frac{2d-p(d-n)}{np}\log(\|\theta^\ast\|_p^p)}}+ \underset{\text{sharpness}}{\underbrace{\frac{1}{n}\log \det(XX^\top)+\frac{p-2}{n}\sum_{j=1}^d\log|\theta^\ast_j|+ K_1(p,d,n)}}
\]
To have a positive regularization term, the dimension of parameter $d$ should not exceed $\frac{p}{p-2}n$. This expression decomposes the criterion into two interpretable parts: a regularization term related to the solution's magnitude, a complexity and sharpness term involving the geometry of the data matrix and the curvature at the minimum.

\subsection{Proof of IIC when $p=1,n=1$}\label{sec:pf_IIC_p1_n1}

When the prior is defined by the $\ell^1$ norm, the resulting density is non-differentiable at the origin, rendering the Laplace approximation unsuitable. However, for certain cases, the marginal likelihood can be evaluated exactly using Fourier analysis. This section details such a derivation for $n=1$.

The normalizing constant of prior when $p=1$ is $c_{1,\tau} = (2\tau)^{-d}$. Recall the integral of interest by Equation \ref{eq:target_integral}, the dual prior is
\begin{equation}
\pi^*(Y) = (2\tau)^{-d}(\det XX^\top)^{-\frac{1}{2}}\int_{X\theta=Y}e^{-\frac{1}{\tau}\|\theta\|_1}\dd \theta.
\label{eq:pi_star_p1}
\end{equation}

We now specialize to the case where $n=1$ and present the main result of this section, which provides an exact expression for the marginal likelihood without approximation on $\tau$.
\begin{lemma}
\label{lemma:pi_p1_n1}
For the case when $n = 1$ with any dimension $d$, $Y,\xi\in\mathbb R$. Assuming $x_j^2\neq x^2_k$ for $j\neq k$, $j,k = 1,...,d$, the dual prior $\pi^*(Y)$ is given by
\[
\pi^*(Y) = \frac{1}{2\tau}\sum_{k=1}^d \frac{1}{|x_k|}e^{-\frac{|Y|}{|x_k|\tau}}\prod_{j=1, j\neq k}^d \frac{x_k^2}{x_k^2-x_j^2}.
\]
\end{lemma}

\begin{proof}
The integral over the affine subspace $\{\theta:X\theta=Y\}$ is the \emph{Radon transform} of the function $f_{\tau}(\theta)=e^{-\frac{1}{\tau}\|\theta\|_1}$. Our strategy is to solve this integral by Fourier transform
\[
\hat f_\tau(\xi)=\int_{\mathbb R^d}e^{i\xi\cdot\theta}f_\tau(\theta)\dd \theta,
\]
which can be inverted by
\[
f_\tau(\theta) = \frac{1}{(2\pi)^d}\int_{\mathbb R^d} e^{-i\theta\cdot\xi}\hat f_\tau(\xi)\dd \xi.
\]

Firstly we compute the Fourier transform of the unnormalized prior kernel $f_{\tau}(\theta)$. Since the exponential sum is separable, its Fourier transform can be factorized as
\begin{align}
\hat{f}_{\tau}(\xi)&= \int_{\mathbb{R}^{d}}e^{i\xi\cdot\theta}e^{-\frac{1}{\tau}\|\theta\|_1}\dd \theta\notag \\
&= \prod_{j=1}^d \int_\mathbb R e^{i\xi_j\theta_j}e^{-\frac{1}{\tau}|\theta_j|}\dd \theta_j.
\label{eq:ft_factor}
\end{align}
We now focus on solving a single one-dimensional integral. Spltting the integration domain and applying Euler's formula yields
\begin{align*}
    \int_\mathbb R e^{i\xi_j\cdot\theta_j}e^{-\frac{1}{\tau}|\theta_j|}\dd \theta
    &= \int_0^\infty e^{-\frac{1}{\tau}\theta_j}(\cos{\xi_j\theta_j}+i\sin\xi_j\theta_j)\dd \theta+ \int_{-\infty}^0 e^{-\frac{1}{\tau}\theta_j}(\cos{\xi_j\theta_j}+i\sin\xi_j\theta_j)\dd \theta\\
    &= 2\int_0^\infty \cos\xi_j\theta_j e^{-\frac{1}{\tau}\theta_j} \dd \theta_j.
\end{align*}
Applying integration by parts twice leads back to the original integral:
\[
\int_0^\infty \cos\xi_j\theta_j e^{-\frac{1}{\tau}\theta_j} \dd \theta_j = \frac{1}{\xi_j^2\tau}-\frac{1}{\xi_j^2\tau^2}\int_0^\infty \cos\xi_j\theta_j e^{-\frac{1}{\tau}\theta_j} \dd \theta_j.
\]
So the integral can be explicitly solved as
\[
\int_\mathbb R e^{i\xi_j\theta_j}e^{-\frac{1}{\tau}|\theta_j|}\dd \theta=\frac{2\tau}{\xi_j^2\tau^2+1}.
\]
Substituting the result for the one dimensional integral back into Equation \ref{eq:ft_factor}, the complete Fourier transform of the prior kernel is
\begin{equation}
\label{eq:f_hat}
\hat{f}_{\tau}(\xi) = \prod_{j=1}^d \int_\mathbb R e^{i\xi_j\theta_j}e^{-\frac{1}{\tau}|\theta_j|}\dd \theta_j
= \prod_{j=1}^d\frac{2\tau}{\xi_j^2\tau^2+1}.
\end{equation}

To calculate the Radon transform of $f_\tau(\theta)$, we denote it as
\begin{equation}
\label{eq:Integral_peq1}
I_{\tau}(Y)=\int_{\theta:\,X\theta=Y}e^{-\frac{1}{\tau}\|\theta\|_{1}}\dd \theta.
\end{equation}
Then the Fourier transform of $I_\tau(Y)$ is
\[
\hat I_\tau(\xi) = \int_{\mathbb R^n} e^{-i\xi\cdot Y}\left(\int_{\theta:X\theta = Y} e^{-\frac{1}{\tau}\|\theta\|_1}\dd \theta\right)\dd Y,
\]
where $\xi\in \mathbb R^n$. Performing the inverse of the Bayesian Duality (Lemma \ref{lemma:bayesian_duality}) projects the integration space into $\mathbb R^d$:
\begin{align*}
\hat I_\tau(\xi) &= \det (XX^\top)^\frac{1}{2}\int_{\mathbb R^d}e^{i\xi\cdot F(\theta)}e^{-\frac{1}{\tau}\|\theta\|_1}\dd \theta\\
&= \det (XX^\top)^\frac{1}{2}\int_{\mathbb R^d}e^{i\xi\cdot (X\theta)}e^{-\frac{1}{\tau}\|\theta\|_1}\dd \theta.
\end{align*}
The dot product in the exponent can be rewritten using properties of the matrix transpose
\[
\xi\cdot (X\theta) = \xi^\top(X\theta) = (X^\top\xi)^\top\theta = (X^\top\xi)\cdot\theta.
\]
This manipulation reveals that the integral is precisely the Fourier transform of $f_\tau(\theta)$ evaluated at the point $X^\top\xi$
\begin{align*}
\hat I_\tau(\xi) &= \det (XX^\top)^\frac{1}{2}\int_{\mathbb R^d}e^{i(X^\top\xi)\cdot\theta}e^{-\frac{1}{\tau}\|\theta\|_1}\dd \theta\\
&= \det (XX^\top)^\frac{1}{2} \hat{f}_{\tau}(X^\top \xi).
\end{align*}
By substituting our previously derived result for $\hat f_\tau$ in Equation \ref{eq:f_hat}, we have 
\[
\hat I_\tau(\xi)= \det (XX^\top)^\frac{1}{2}\prod_{j=1}^{d}\frac{2\tau}{1+(x_j\cdot \xi)^{2}\tau^{2}}
\]
Here $x_j\in\mathbb R^n$ denotes the $j$-th column of $X$ and $x_j\cdot\xi = (X^\top\xi)_j$. Then by the inverse Fourier formula, 
\[
I_{\tau}(Y)=\frac{\det (XX^\top)^\frac{1}{2}}{(2\pi)^{n}}\int_{\mathbb{R}^{n}}e^{-iY\cdot\xi}\prod_{j=1}^{d}\frac{2\tau}{1+(x_{j}\cdot\xi)^{2}\tau^2}\dd \xi.
\]
Substituting this expression for the integral back into $\pi^*(Y)$ in Equation \ref{eq:pi_star_p1} and simplifying the constants, we obtain
\begin{align*}
\pi^*(Y) &= (2\tau)^{-d}(\det XX^\top)^{-\frac{1}{2}}\int_{X\theta=Y}e^{-\frac{1}{\tau}\|\theta\|_{1}}\dd \theta \\
&= \frac{1}{(2\pi)^n}\int_{\mathbb R^n} e^{-iY\cdot\xi}\prod_{i=1}^d\frac{1}{1+(x_i\cdot \xi)^2\tau^2}\dd \xi.
\end{align*}

When $n=1$, the integral simplifies to one dimension
\begin{equation}
\pi^*(Y) = \frac{1}{2\pi}\int_{-\infty}^\infty e^{-iY\xi}\prod_{j=1}^d \frac{1}{1+(x_j\xi\tau)^2}\dd \xi
\label{eq:pi_star_inverse_fourier}
\end{equation}
This integral can be evaluated using the residue theorem from complex analysis. The integrand is meromorphic with $2d$ simple poles. 

Our approach depends on the sign of $Y$. Let us first consider the case when $Y>0$. To ensure the convergence of the integral, we close the integration contour in the lower half of the complex plane with a semicircle $\Gamma$ with radius $R$. This choice is motivated by the term $e^{-iY\xi}$, which decays to zero when $\textrm{Im}(\xi)<0$. The poles located inside this contour are $\xi_k = -\frac{1}{|x_k|\tau} i$ for $k=1,...,d$.

Define the integrand inside $\pi^*(Y)$ as
\[
J(\xi) = e^{-iY\xi}\prod_{j=1}^d \frac{1}{1+(x_j\xi\tau)^2}.
\]
Then the residue corresponding to the pole $\xi_k$ is calculated as 
\begin{align*}
\text{Res} (J(\xi_k)) &= \lim_{\xi\to\xi_k}(\xi-\xi_k)e^{-iY\xi}\prod_{j=1}^d \frac{1}{1+(x_j\xi\tau)^2}\\
&= \lim_{\xi\to\xi_k}\frac{1}{|x_k|\tau(|x_k|\xi\tau-i)}e^{-iY\xi}\prod_{j=1,j\neq k}^d \frac{1}{1+(x_j\xi\tau)^2}\\
&= -\frac{1}{2|x_k|\tau i}e^{-\frac{Y}{|x_k|\tau}}\prod_{j=1, j\neq k}^d \frac{1}{1-x_j^2/x_k^2}\\
&= -\frac{1}{2|x_k|\tau i}e^{-\frac{Y}{|x_k|\tau}}\prod_{j=1, j\neq k}^d \frac{x_k^2}{x_k^2-x_j^2}.
\end{align*}
By the residue theorem, the integral over the closed contour is equal to $-2\pi i$ times the sum of the residues, where the negative sign accounts for the clockwise orientation of the contour. The contour can be separated into a semicircle $\Gamma$ and the straight line along the real axis
\[
-2\pi i\sum_{k=1}^d \text{Res} (J(\xi_k)) = \int_\Gamma e^{-iY\xi}\prod_{j=1}^d \frac{1}{1+(x_j\xi\tau)^2}\dd \xi + \int_{-R}^R e^{-iY\xi}\prod_{j=1}^d \frac{1}{1+(x_j\xi\tau)^2}\dd \xi.
\]
The integral over the semicircular arc $\Gamma$ vanishes as $R\to \infty$ because the magnitude of the integrand decays.
\begin{align*}
\left|\int_\Gamma e^{-iY\xi}\prod_{j=1}^d \frac{1}{1+(x_j\xi\tau)^2}\dd \xi\right|
&\le\int_\Gamma\left|e^{-iY\xi}\prod_{j=1}^d \frac{1}{1+(x_j\xi\tau)^2}\right|\dd \xi \\
&\le \pi R\sup_\Gamma \left|\prod_{j=1}^d \frac{e^{-iY\xi}}{1+(x_j\xi\tau)^2}\right|\dd \xi\\
&\le \frac{\pi R}{((x_j\tau)^2 R^2-1)^d}\to 0
\end{align*}
as $R\to \infty$. Then 
\begin{align*}
\lim_{R\to\infty}\int_{-R}^R e^{-iY\xi}\prod_{j=1}^d \frac{1}{1+(x_j\xi\tau)^2}\dd \xi &= \int_{-\infty}^\infty e^{-iY\xi}\prod_{j=1}^d \frac{1}{1+(x_j\xi\tau)^2}\dd \xi\\
&= -2\pi i\sum_{k=1}^d \text{Res} (J(\xi_k)).
\end{align*}
Substituting the residual evaluation into $\pi^*(Y)$ in Equation \ref{eq:pi_star_inverse_fourier} for $Y>0$, we find
\begin{align*}
\pi^*(Y) &= -\frac{1}{2\pi}\left(2\pi i\sum_{k=1}^d \text{Res} (J(\xi_k))\right)\\
&= \frac{1}{2\tau}\sum_{k=1}^d \frac{1}{|x_k|}e^{-\frac{Y}{|x_k|\tau}}\prod_{j=1, j\neq k}^d \frac{x_k^2}{x_k^2-x_j^2}
\end{align*}
For the case when $Y<0$, a similar argument is applied by closing the contour in the upper half-plane, which encloses the poles at $\xi_k=+\frac{1}{|x_k|\tau}i$. This procedure yields a result where $Y$ is replaced by $-Y$. Both cases can be combined into a single expression by using the absolute value $Y$, which completes the proof.
\end{proof}

Following the derivation of the dual prior $\pi^*(Y)$, we now proceed to calculate the corresponding Bayes free energy and the Interpolating Information Criterion. Taking the negative logarithm of the expression for $\pi^*(Y)$ presents a challenge due to the summation within the formula. A direct analytical simplification of the logarithm of a sum is generally not feasible. However, in the asymptotic regime of a small regularization parameter $\tau$, the sum is dominated by the term with the slowest exponential decay
\[
\pi^*(Y)\sim\frac{1}{2\tau}\frac{1}{\|x\|_{\infty}}e^{-\frac{|Y|}{\|x\|_\infty\tau}}\prod_{j=1, j\neq I}^d \frac{x_I^2}{x_I^2-x_j^2},
\]
where $I$ is the index of the maximum component in $X$. Taking the negative logarithm, we have
\[
-\log \pi^*(Y)\sim \log(2\tau) + \log(\|x\|_{\infty})+\frac{|Y|}{\|x\|_\infty\tau}-\log \prod_{j=1, j\neq I}^d \frac{x_I^2}{x_I^2-x_j^2}.
\]
Taking the derivative over $\tau$ gives the minimum point
\[
\tau^* = \frac{|Y|}{\|x\|_\infty}.
\]
This is the optimal parameter $\theta^\ast$ because by Holder's inequality, 
\[
|Y|=|x\cdot \theta|\le \|x\|_\infty\|\theta\|_1.
\]
Then $\|\theta\|_1\ge |Y|/\|x\|_\infty$ so the minimum point $\|\theta^\ast\|_1 = |Y|/\|x\|_\infty$. Therefore the IIC in this case is
\[
\text{IIC} = \underset{\mathrm{regularization}}{\underbrace{2 \log\|\theta^\ast\|_1}}-\underset{\mathrm{sharpness}}{\underbrace{2 \log \left(\frac{1}{2\|x\|_{\infty}}\prod_{j=1, j\neq I}^d \frac{x_I^2}{x_I^2-x_j^2}\right)}},
\]
where we once again decompose the information criterion into a regularization-sharpness tradeoff as in the $\ell^p$ case.

\subsection{Proof of Theorem \ref{thm:iic_peq1}: IIC with $p=1$}\label{sec:pf_IIC_p1}
\begin{theorem}
\label{thm:int_approx_peq1}
    Suppose that $R(\cdot)$ is first-order convex and has a unique global minimizer $\theta^\ast$. If the integral $V = \int_{\ker X} \exp\{-\nabla_z R(\theta^\ast)\}\dd z$ is finite, where $\nabla_zR(\theta^\ast)$ denotes the directional derivative of $R$ at $\theta^\ast$ in the direction $z$,
    then the asymptotic approximation for this integral is
    \[
    I_\tau(Y) = \int_{X\theta=Y}e^{-\frac{1}{\tau}R(\theta)}\dd \theta\sim \tau^{d-n}V e^{-\frac{1}{\tau}R(\theta^\ast)}
    \]
    as $\tau \to 0^+$.
\end{theorem}
\begin{proof}
The integration domain is the affine subspace $M=\{\theta:X\theta = Y\}$. Any $\theta \in M$ can be parameterized by the unique minimizer $\theta^\ast$ and a vector $v$ from the null space of $X$, such that $\theta = \theta^\ast + v$. We perform a change of variables $v = \tau z$. Since $\ker X$ is a linear subspace, the change of variables maps $\ker X$ onto itself. The integral becomes
\begin{align*}
I_\tau(Y) &= \int_{\ker(X)} e^{-\frac{1}{\tau}R(\theta^\ast+v)}dv\\
&= \tau^{d-n}e^{-\frac{1}{\tau}R(\theta^\ast)} \int_{\ker(X)} \exp\left\{-\frac{R(\theta^\ast+\tau z)-R(\theta^\ast)}{\tau}\right\}\dd z.
\end{align*}
To evaluate the limit of the integral as $\tau \to 0^+$, we apply the Dominated Convergence Theorem~\citep[Theorem 1.13]{stein_real_2005}. 

Firstly, we aim to find pointwise convergence of the integrand. By the definition of the directional derivative, we have
\[
\lim_{\tau\to0+}\frac{R(\theta^\ast+\tau z)-R(\theta^\ast)}{\tau}=\nabla_z R(\theta^\ast).
\]
Since the exponential function is continuous, the integrand converges pointwise
\begin{equation}
\label{eq:DCT_gradient_C1}
    \lim_{\tau\to 0^+}\exp\left\{-\frac{R(\theta^\ast+\tau z)-R(\theta^\ast)}{\tau}\right\}
    = \exp\left\{-\nabla_zR(\theta^\ast)\right\}.
\end{equation}

Secondly, to find a dominating function, we use the convexity of the norm. The definition of a subgradient implies that for any $g\in\partial R(\theta^\ast)$, 
\[
R(\theta^\ast+\tau z)\ge R(\theta^\ast) + g\cdot (\tau z)
\]
Hence
\[
\left |\exp\left\{-\frac{R(\theta^\ast+\tau z)-R(\theta^\ast)}{\tau}\right\}\right | \le \exp\left\{-\sup_{g\in\partial R(\theta^\ast)}g\cdot z\right\} = \exp\left\{\nabla_z R(\theta^\ast)\right\}.
\]
By hypothesis, the integral is finite and equal to $V$. 
Combining this with\ref{eq:DCT_gradient_C1}
implies the conditions of the dominated convergence theorem are satisfied, thus we can interchange the order of the limit and the integral giving
\begin{align*}
I(Y) &\sim \tau^{d-n}e^{-\frac{1}{\tau}R(\theta^\ast)} \int_{\ker(X)} \exp\left\{-\nabla_zR(\theta^\ast)\right\}\dd z\\
&\sim \tau^{d-n}V e^{-\frac{1}{\tau}R(\theta^\ast)}.
\end{align*}
\end{proof}

The validity of the approximation in Theorem~\ref{thm:int_approx_peq1} relies on the properties of the directional derivative $\nabla_z R(\theta^\ast)$ and the finiteness of the sharpness integral $V$. We now establish these properties in the $\ell^1$ case. Before proceeding to the following proposition, we first give the one-sided directional derivative of the $\ell^1$ regularizer $R(\theta) = \|\theta\|_1$ at $\theta^\ast$ along the direction $z$. It is given by the subgradient set with the max formula~\citep[Theorem 17.19]{bauschke_convex_2011}
\[
\nabla_z R(\theta^\ast) = \sup_{g\in\partial R(\theta^\ast)} g^\top z.
\]
Let $S = \{j:\theta^\ast_j\neq 0\}$ is the support set of $\theta^\ast$ and $C$ is the complement of $S$, i.e. $C = \{j: \theta_j^* = 0\}$. The subgradient set $\partial R(\theta^\ast)$ is given by 
\[
\left\{\partial R(\theta^\ast)\right\}_j = 
\begin{cases}
    \mathrm{sgn}(\theta^\ast_j), \quad & j\in S\\
    [-1,1],&j\in C
\end{cases}
\]
Let $\theta_S$ be the subvector of $\theta$ formed by taking elements $\theta_j$ with $j \in S$, and similarly for the direction vector $z_S$ and $z_C$. The sign vector is defined as $s  = \mathrm{sgn}(\theta_S^*)$. The supremum is achieved by taking the absolute value of the subgradient components on the inactive set $C$, which leads to the expression
\begin{equation}
\label{eq:expicit_directional_directive}
    \nabla_zR(\theta^\ast) = \|z_C\|_1 + s\cdot z_S.
\end{equation}

\begin{proposition}
\label{prop: derivative ineq}
    Let $R(\theta) = \|\theta\|_1$. Under Assumption \ref{assum:unique}, there exists a positive constant $\alpha >0$, such that 
    \[
    \nabla_zR(\theta^\ast)\ge\alpha \|z\|_1
    \]
    for all $z\in \ker X$.
\end{proposition}
\begin{proof}
Let $X_S$ be the submatrix of $X$ formed by taking columns $x_j$ with $j \in S$, and similarly for $X_C$ with $j \in C$. As shown by Zhang et al.~\citep[Theorem 2]{zhang_necessary_2015}, Assumption~\ref{assum:unique} implies that the matrix $X_S$ has full column rank, and there exists a vector $\mu\in\mathbb R^n$ such that $X_S^\top \mu = s$ and $\|X^\top_C \mu\|_\infty <1$. Applying these properties to (\ref{eq:expicit_directional_directive}), we have 
\[
\nabla_zR(\theta^\ast) = \|z_C\|_1 + \mu^\top X_Sz_S.
\]
Since $z\in \ker X$, we have $X_Sz_S + X_Cz_C = Xz = 0$. Then
\begin{align*}
\nabla_zR(\theta^\ast) &= \|z_C\|_1 - \mu^\top X_Cz_C\\
&\ge (1-\|X_C^\top \mu\|_\infty)\|z_C\|_1.
\end{align*}
The final inequality follows from Holder's inequality~\citep[Theorem 1.1]{stein_functional_2011}:
\[
|\mu^\top X_C z_C| = |(X_C^\top \mu)^\top \cdot z_C|\le \|(X_C^\top \mu)^\top\|_\infty \|z_C\|_1.
\]
Let $\beta = 1-\|X_C^\top \mu\|_\infty > 0$. Thus we have shown that $\nabla_z\|\theta^\ast\|_1 \ge \beta\|z_C\|_1$. The proof is completed by the following lemma, which relates $\|z_C\|_1$ to the full norm $\|z\|_1$. By Lemma~\ref{lemma:ineq_zc_z}, we denote $\alpha = \beta\rho>0$ and have the desired result
\[
\nabla_z\|\theta^\ast\|_1 \ge \beta \|z_C\|_1\ge \beta\gamma\|z\|_1 = \alpha \|z\|_1.
\]
\end{proof}

\begin{lemma}
\label{lemma:ineq_zc_z}
    If $X_S$ has full column rank, there exist a constant $\rho>0$, such that for all $z\in\ker X$, $\|z_C\|_1 \ge \rho\|z\|_1$.
\end{lemma}
\begin{proof}
    For $\|z\|_1 = 0$ this inequality is trivial. For any non-zero $z\in\ker X$, we can normalize it and consider $\tilde z=z/\|z\|_1$. The inequality is homogeneous, so it suffices to prove that $\|z_C\|_1>\rho$ for some $\rho>0$ with all $z\in P$ where set $P$ is defined as $P=\{z\in\ker X:\|z\|_1=1\}$. This set $P$ is bounded and closed, so it is a compact set. Define 
    \[
    \rho = \inf_{z\in \ker X}\left\{ \|z_C\|_1: \|z\|_1 = 1 \right\}.
    \]
    By the Extreme Value Theorem~\citep[Theorem 4.16]{rudin_principles_2008}, a continuous function on a compact set must attain its minimum value. Therefore, the infimum is achieved. Now we prove $\rho>0$ by contradiction. Assume $\rho = 0$. Then there exists a sequence $z^*\in P$, such that $\|z_C^*\|_1 \to 0$. Since $z^*\in \ker X$, we have $Xz^* = 0 = X_Sz^*_S + X_Cz^*_C$. As $\|z^*_C\|_1 \to 0$, it follows that $z^*_C \to 0$ by the norm is non-negative. Consequently, the product with matrix $X_C z^*_C \to 0$ and thus $X_Sz^*_S \to 0$. Given that $X_S$ has full column rank, its null space is trivial, which implies $z^*_S \to 0$. This leads to $\|z^*\|_1 = \|z_S\|_1+\|z_C\|_1 \to 0$, which contradicts to the assumption $\|z^*\|_1 = 1$. Thus $\rho$ must be positive.
\end{proof}

Proposition \ref{prop: derivative ineq} together with Lemma \ref{lemma:ineq_zc_z} gives an upper bound of integral $V$ by
\begin{equation}
\label{eq:VUpperBound}
V \leq \int_{\ker X} e^{-\alpha \|z\|_1} \dd z.
\end{equation}
The following lemma will prove that this upper bound is integrable.
\begin{lemma}
    \label{lemma: converge integral}
    For any constant $\alpha >0$, the integral
    \[
    \int_{\ker X} e^{-\alpha\|z\|_1} \dd z < +\infty.
    \]
\end{lemma}
\begin{proof}
    Let $Q \in \mathbb R^{d\times(d-n)}$ be an orthonormal basis for $\ker X$, such that any $z\in\ker X$ can be written as $z = Qw$ for some $w \in \mathbb R^{d-n}$. By $Q^\top Q=I_{d-n}$, the integral transforms to
    \[
    \int_{\ker X} e^{-\alpha\|z\|_1} \dd z = \int_{\mathbb R^{d-n}} e^{-\alpha\|Qw\|_1}\dd w.
    \]
    Using the norm inequality $\|w_2\|\le\|w_1\|$,
    \begin{align*}
    \int_{\ker X} e^{-\alpha\|z\|_1} \dd z &\le \int_{\mathbb R^{d-n}}e^{-\alpha\|Qw\|_2}\dd w\\
        &= \int_{\mathbb R^{d-n}}e^{-\alpha\|w\|_2}\dd w
    \end{align*}
    By Cauchy–Schwarz inequality $\frac{1}{\sqrt{d-n}}\|w\|_1 \le\|w\|_2$, 
    \begin{align*}
        \int_{\ker X} e^{-\alpha\|z\|_1} \dd z 
        &\le \int_{\mathbb R^{d-n}}e^{-\frac{\alpha}{\sqrt{d-n}}\|w\|_1}\dd w\\
        &= \prod_{i=1}^{d-n} \int_\mathbb R e^{-\frac{\alpha}{\sqrt{d-n}}|w_i|}\dd w_i\\
        &= \left(\frac{2\sqrt{d-n}}{\alpha}\right)^{d-n} <\infty
    \end{align*}
\end{proof}

With the asymptotic form of the integral established, we can now compute the Bayes free energy and the IIC. First, we evaluate the sharpness integral
\[
V = \int_{\ker X} \exp\{-\nabla_z R(\theta^\ast)\}\dd z = \int_{\ker X} \exp\{-\|z_C\|_1 - s\cdot z_S\}\dd z.
\]
This integral can be simplified by layer cake representation with the basic fact $e^{-t} = \int_\mathbb R e^{-u}\mathds 1\{u\ge t\} \dd u$. Applying Fubini's theorem~\citep[Theorem 3.1]{stein_real_2005}, we get
\begin{align*}
V &= \int_{\ker X}\int_{\mathbb R} e^{-u}\mathds 1\left\{u \ge \|z_C\|_1 + s\cdot z_S\right\}\dd u \dd z\\
&= \int_\mathbb R e^{-u}\int_{\ker X} \mathds 1\left\{u \ge \|z_C\|_1 + s\cdot z_S\right\}\dd z \dd u
\end{align*}
By a change of variables $z'=z/u$
\begin{align*}
V &= V_0\int_\mathbb R u^{d-n}e^{-u} \dd u\\
&= (d-n)!V_0
\end{align*}
Inside this simplification, we define
\begin{equation}
\label{eq:v0}
    V_0 := \int_{\ker X} \mathds 1\left\{\|z'_C\|_1 + s\cdot z'_S\le 1\right\}\dd z
\end{equation}
as the sharpness of the convex set $K_0 = \left\{\mathds 1\left\{\|z'_C\|_1 + s\cdot z'_S\le 1\right\}: z'\in \ker X\right\}$.

\textbf{Following the dual prior.}
This result gives a simplification of Theorem~\ref{thm:int_approx_peq1} to
\[
I(y) \sim (d-n)!\tau^{d-n}e^{-\frac{1}{\tau}R(\theta^\ast)}V_0.
\]
Substituting this back into the approximation for $\pi^*(Y)$ in Equation \ref{eq:pi_star_p1}, when $\tau\to0^+$ we have
\[
\pi^*(y) \sim 2^{-d}(d-n)!\tau^{-n}\det(XX^\top)^{-1/2}e^{-\frac{1}{\tau}R(\theta^\ast)}V_0.
\]
By Lemma \ref{lemma:duality_asymptotic}, take $\gamma\to 0^+$ and then take $\tau\to 0^+$, the asymptotic Bayes free energy $\bar{\mathcal F}_{n,\tau} \sim -\log Z_n$ is therefore
\[
\bar{\mathcal F}_{n,\tau} = -\log \left(2^{-d}(d-n)!\det(XX^\top)^{-1/2}V_0\right)+n\log\tau + \frac{R(\theta^\ast)}{\tau}.
\]
Differentiating in $\tau$ and setting the derivative to zero yields
\begin{align*}
    \frac{\partial \bar{\mathcal F}_{n,\tau}}{\partial\tau} &= \frac{n}{\tau}-\frac{R(\theta^\ast)}{\tau^2} = 0,\\
    \tau^* &= \frac{1}{n}R(\theta^\ast).
\end{align*}
Let $R(\theta^\ast)=\|\theta^\ast\|_1$. Substituting $\tau^*$ back into the scaled free energy $\frac{2}{n}\bar{\mathcal F}_{n,\tau}$ in Definition \ref{def:IIC} gives the IIC decomposition
\begin{equation}
    \text{IIC} \simeq \underset{\text{regularization}}{\underbrace{2\log \|\theta^\ast\|_1}} -\underset{\text{sharpness}}{\underbrace{\frac{2}{n}\log (V_0) + \frac{1}{n}\log\det(XX^\top)+K_2(d,n)}}.
\end{equation}
where 
\begin{equation}
\label{eq:K2}
    K_2(d,n) = -2\log n  -\frac{2}{n}\log \left(2^{-d}(d-n)!\right).
\end{equation}

This expression reveals the regularization-sharpness tradeoff for LASSO and resembles the $n = 1$ case. The regularization term penalizes solutions with large $\ell^1$ norm. The sharpness term captures both the sharpness spanned by the data points and the sharpness of the solution spaces with directions of the $\ell^1$ norm. 

\subsubsection{Computing $V_0$}

It now only remains to compute the volume term $V_0$ given by
\[
V_0 = \int_{\ker X} \mathds 1 \{ \|z_C\|_1 + s\cdot z_S \leq 1\} \dd z,
\]
where $S = \mbox{supp}(\theta^\ast)$, $C = \{1,\dots,d\}\setminus S$, and $s = \mbox{sign}(\theta_S^\ast)$. Now, we assume that $|S| = n$, so that $z = (z_S, z_C) \in \mathbb{R}^n \times \mathbb{R}^{d-n}$. Because $z \in \ker(X)$, it follows that
\[
X_S z_S + X_C z_C = 0\text{ and so } z_S =-X_S^{-1} X_C z_C.
\]
Defining
\[
\Psi = X_S^{-1} X_C \in \mathbb{R}^{n \times (d-n)},
\]
then we have that $z = (-\Psi z', z')$ for $z' \in \mathbb{R}^{d-n}$ and
\[
\|z_C\|_1 + s\cdot z_S = \|z'\|_1 + s^\top (-\Psi z') = \sum_{k=1}^{d-n} (|z_k| - \psi_k z_k),
\]
where $\psi = \Psi^\top s$. Therefore, we have, by a change of variables,
\[
V_0 = \sqrt{\det(I + \Psi^\top \Psi)} \int_{\mathbb{R}^{d-n}} \mathds 1\left\{\sum_{k=1}^{d-n} (|z_k'| - \psi_k z_k') \leq 1 \right\} \dd z'.
\]
Now we can decompose the integral by orthants: let $z_k' = \epsilon_k u_k$ for $u_k \geq 0$,
\[
V_0 = \sqrt{\det(I + \Psi^\top \Psi)} \sum_{\epsilon \in \{\pm 1\}^{d-n}} \int_{\mathbb{R}^{d-n}_+} \mathds 1\left\{\sum_{k=1}^{d-n}(1 - \psi_k \epsilon_k)u_k \leq 1\right\} \dd u.
\]
For every integral to be finite, it is necessary and sufficient that $1 - \psi_k > 0$ and $1 + \psi_k > 0$, so $|\psi_k| < 1$ for each $k$. From (\ref{eq:VUpperBound}) and Lemma \ref{lemma: converge integral}, $V_0$ is finite, and so $|\psi_k| < 1$ for every $k$. The rest of the proof will follow from the following Lemma \ref{lem:SimplexVol}, providing a formula for the rescaled simplex. 
\begin{lemma}
\label{lem:SimplexVol}
For any constants $c_1,\dots,c_m > 0$, 
\[
\mathrm{Vol}\left\{ (v_{1},\dots,v_{m}):v_{k}\geq0,\sum_{k=1}^{m}c_{k}v_{k}\leq1\right\} =\frac{1}{m!c_{1}\cdots c_{m}}.
\]
\end{lemma}
\begin{proof}
Let $\Delta_m$ denote the standard unit simplex given by
\[
\Delta_m = \left\{(v_1,\dots,v_m):v_k \geq 0, \sum_{k=1}^m v_k \leq 1\right\}.
\]
By a change of variables, it is evident that
\[
\mathrm{Vol}\left\{ (v_{1},\dots,v_{m}):v_{k}\geq0,\sum_{k=1}^{m}c_{k}v_{k}\leq1\right\} = \frac{\mathrm{Vol}(\Delta_m)}{c_1\cdots c_m},
\]
so it suffices to compute the volume of $\Delta_m$. This is a classical result; here we reproduce a simple proof for completeness. It is clear that $\mathrm{Vol}(\Delta_1) = 1$, but also
\begin{align*}
\mathrm{Vol}(\Delta_m) &= \int_0^1 \mathrm{Vol}((1-s) \Delta_{m-1}) \dd s \\
&= \mathrm{Vol}(\Delta_{m-1}) \cdot \int_0^1 (1-s)^{m-1} \dd s, \\
&= \frac{1}{m} \mathrm{Vol}(\Delta_{m-1}),
\end{align*}
which implies $\mathrm{Vol}(\Delta_m) = \frac{1}{m!}$. 
\end{proof}
From Lemma \ref{lem:SimplexVol},
\begin{align*}
V_{0}&=\frac{\sqrt{\det(I+\Psi^{\top}\Psi)}}{(d-n)!}\cdot\sum_{\epsilon\in\{\pm1\}^{d-n}}\prod_{k=1}^{d-n}\frac{1}{1-\psi_{k}\epsilon_{k}}\\
&=\frac{\sqrt{\det(I+\Psi^{\top}\Psi)}}{(d-n)!}\cdot\prod_{k=1}^{d-n}\left(\frac{1}{1-\psi_{k}}+\frac{1}{1+\psi_{k}}\right)\\
&=\frac{2^{d-n}\sqrt{\det(I+\Psi^{\top}\Psi)}}{(d-n)!}\cdot\prod_{k=1}^{d-n}\frac{1}{1-\psi_{k}^{2}}.
\end{align*}

\section{The Special Case $p=1,\ n=1$}
\begin{figure}
    \centering
    \includegraphics[width=0.5\linewidth]{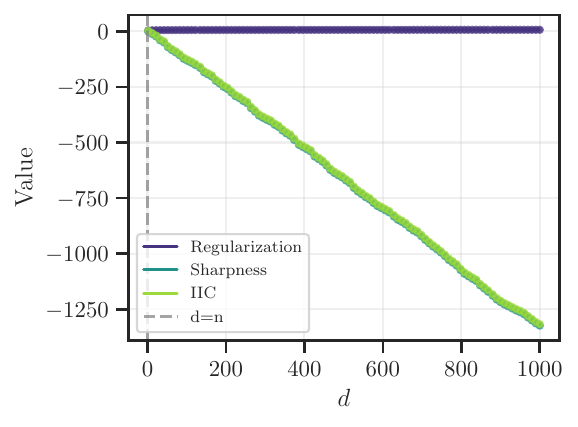}
    \caption{Decomposition of the Interpolating Information Criterion (green) for minimum $\ell^1$-norm interpolating solutions using random Fourier features as a tradeoff between the effect of regularization (purple) and local sharpness (blue). This plot uses the FLIR dataset and randomly select one sample to show the particular result in Corollary \ref{thm:iic_p1_n1}.}
    \label{fig:IIC_p1_n1}
\end{figure}

In the main experiments we plot the $p=1$ decomposition only when the fitted support has the same cardinality as the data set. To further illustrate the change in behavior observed at $p=1$, we additionally consider the small sample setting $n=1$, as derived in Corollary \ref{thm:iic_p1_n1}. In this setting, the sharpness term can decrease much more noticeably, which helps show the maximum extent to which the $\ell^1$ penalty can trade off a weaker decrease in the regularization term for a stronger decrease in sharpness.

Figure \ref{fig:IIC_p1_n1} reports the decomposition for $p=1,\ n=1$ on the FLIR dataset (same preprocessing and RFF construction as in the main text). The purpose of this figure is not to claim practical relevance for $n=1$, but to provide a clear diagnostic example demonstrating how sparsity induced by $\ell^1$ can lead to a substantial reduction in the sharpness component compared to the higher $p$ cases discussed in Figure \ref{fig:IIC_RFF}.

\end{document}